\documentclass{amsart} 
\usepackage{amssymb}
\usepackage{fullpage}

\usepackage[foot]{amsaddr}
\usepackage[bookmarks=true,hypertexnames=false,pagebackref]{hyperref}
\hypersetup{colorlinks=true, citecolor=blue, linkcolor=red, urlcolor=blue}
\usepackage[T1]{fontenc}
\usepackage[nott,sfmathbb]{kpfonts}
\usepackage{textcomp} 
\usepackage[scr=rsfso]{mathalfa}
\usepackage{bm}
\usepackage{varwidth}
\usepackage{todonotes}

\usepackage{graphicx}
\usepackage{physics}
\usepackage{xifthen}
\usepackage{booktabs}
\usepackage{multirow}
\usepackage{array}
\usepackage{makecell}
\allowdisplaybreaks[4] 
\usepackage{tikz}
\usepackage{rotating}
\usetikzlibrary{math}
\usetikzlibrary{positioning}
\usepackage[justification = centering, labelsep =period]{caption}
\usepackage{float}
\usepackage[boxed]{algorithm2e}
\LinesNumbered
\SetKwInOut{Input}{Input}
\SetKwInOut{Output}{Output}

\usepackage{cleveref} 

\newtheorem{theorem}{Theorem}
\newtheorem{lemma}[theorem]{Lemma}
\newtheorem{corollary}[theorem]{Corollary}

\newtheorem{conjecture}[theorem]{Conjecture}

\newtheorem{proposition}[theorem]{Proposition}

\def\*#1{\mathbf{#1}} \def\+#1{\mathcal{#1}} \def\-#1{\mathrm{#1}}\def\^#1{\mathbb{#1}}\def\!#1{\mathtt{#1}}
\newcommand{\set}[1]{\left\{#1\right\}}
\newcommand{\tuple}[1]{\left(#1\right)} 
\newcommand{\inner}[2]{\langle #1,#2\rangle} \newcommand{\tp}{\tuple}
\renewcommand{\mid}{\;\middle\vert\;} \newcommand{\cmid}{\,:\,}
 
\newcommand{\defeq}{:=} 
\newcommand{\ol}{\overline}
\newcommand{\argmin}{\mathop{\arg\min}}

\renewcommand{\Pr}[2][]{ \ifthenelse{\isempty{#1}}
  {\mathbf{Pr}\left[#2\right]} {\mathbf{Pr}_{#1}\left[#2\right]} }
\newcommand{\E}[2][]{ \ifthenelse{\isempty{#1}}
  {\mathbf{E}\left[#2\right]}
  {\mathbf{E}_{#1}\left[#2\right]} }
\newcommand{\Var}[2][]{ \ifthenelse{\isempty{#1}}
  {\mathbf{Var}\left[#2\right]}
  {\mathbf{Var}_{#1}\left[#2\right]} }
\newcommand{\Ent}[2][]{ \ifthenelse{\isempty{#1}}
  {\mathbf{Ent}\left[#2\right]}
  {\mathbf{Ent}_{#1}\left[#2\right]} }

\renewcommand{\emptyset}{\varnothing}

\newcommand{\Nout}{N_{\!{out}}}
\newcommand{\Nin}{N_{\!{in}}}
\newcommand{\wh}{\widehat}

\title[Improved Algorithms for Bandit with Graph Feedback]{Improved Algorithms for Bandit with Graph Feedback via Regret Decomposition}

\author{Yuchen He}
\address[Yuchen He]{Shanghai Jiao Tong University, China. \textnormal{E-mail: \url{yuchen\_he@sjtu.edu.cn}}}
\author{Chihao Zhang}
\address[Chihao Zhang]{Shanghai Jiao Tong University, China. \textnormal{E-mail: \url{chihao@sjtu.edu.cn}}}

\begin{document}

\maketitle

\begin{abstract}
  The problem of bandit with graph feedback generalizes both the multi-armed bandit (MAB) problem and the learning with expert advice problem by encoding in a directed graph how the loss vector can be observed in each round of the game. The mini-max regret is closely related to the structure of the feedback graph and their connection is far from being fully understood. We propose a new algorithmic framework for the problem based on a partition of the feedback graph. Our analysis reveals the interplay between various parts of the graph by decomposing the regret to the sum of the regret caused by small parts and the regret caused by their interaction. As a result, our algorithm can be viewed as an interpolation and generalization of the optimal algorithms for MAB and learning with expert advice. Our framework unifies previous algorithms for both strongly observable graphs and weakly observable graphs, resulting in improved and optimal regret bounds on a wide range of graph families including graphs of bounded degree and \emph{strongly observable graphs with a few corrupted arms}.
\end{abstract}
\section{Introduction}
\emph{Multi-armed bandit} (MAB) and \emph{learning with expert advice} are two canonical models in online learning and have been extensively studied in recent years. Both games proceed for $T$ rounds. In each round, the player can pull one of $N$ arms and the (adversarial) environment decides the loss of each arm. In MAB, the player can only observe the loss of the arm just pulled while in the model of learning with expert advice, the whole loss vector is visible. The goal of the player is to pull arms so that the cumulative loss in $T$ rounds is minimized. The performance of a player is usually measured by the notion of mini-max regret $R^*(T)$, the expected gap between the loss of the player's strategy and the loss of the best fixed arm against the worst loss vectors. 

Bandit with graph feedback generalizes both models in terms of the fraction of the loss vector that can be observed in each round. The $N$ arms can be viewed as the vertices in a directed feedback graph $G=(V,E)$, indexed by $\set{1,2,\dots,N}$ and an edge $(i,j)$ indicates if the arm $i$ is pulled, the loss at arm $j$ can be observed. Therefore, MAB corresponds to the case when $G$ consists of $N$ isolated vertices with self-loops, and learning with expert advice, sometimes called the full feedback model, corresponds to the case when $E=V^2$.

Tight bounds of the mini-max regret for both MAB and learning with expert advices are known. It was shown in~\cite{ACBFS02} and~\cite{FS97} that the optimal regret of two models are $\Theta\big(\tp{N\cdot T}^{\frac{1}{2}}\big)$ and $\Theta\big(\tp{\log N\cdot T}^{\frac{1}{2}}\big)$ respectively. The difference between the two regret bounds is clearly due to the amount of information the player can gather about the loss vectors. As a result, the work of~\cite{MS11} initialized the study of regret with graph feedback. 

This line of research was further extended in the work of~\cite{ACBDK15}, which classifies all graphs into three classes: non-observable graphs, strongly observable graphs and weakly observable graphs. A non-observable graph contains arms that can never be observed and thus suffers $\Theta(T)$ regret. Strongly observable graphs are interpolation of MAB and learning with expert advice so that each vertex either has a self-loop or can be observed by all other arms. The mini-max regrets of these graphs are $O\big(\tp{\alpha(G)\cdot T}^{\frac{1}{2}}\cdot \log\tp{NT}\big)$ where $\alpha(G)$ is the \emph{independence number} of $G$. The remaining graphs are called \emph{weakly observable} and it was shown that their regret is $O\big(\tp{\delta(G)\log N}^{\frac{1}{3}}\cdot T^{\frac{2}{3}}\big)$ where the $\delta(G)$ is the \emph{domination number} of $G$. The bound has been recently improved  to $O\big(\tp{\delta^*(G)\log N}^{\frac{1}{3}}\cdot T^{\frac{2}{3}}\big)$ in~\cite{CHLZ21} where $\delta^*(G)$ is the \emph{fractional dominating number} of $G$ satisfying $\delta^*(G)\le \delta(G)$. The ultimate goal in this line of research is to answer the following question:

\smallskip
\begin{quote}
	\emph{How the structure of the feedback graph affects the mini-max regret?}	
\end{quote}

\smallskip
Unfortunately, all previous results are not optimal even on very simple feedback graphs. Consider an undirected cycle with $2N$ vertices. We have $\delta(G)=\delta^*(G)=N$ and therefore previous algorithms have regret $O\big(\tp{N\log N}^{\frac{1}{3}}\cdot T^{\frac{2}{3}}\big)$. On the other hand, it was shown in~\cite{CHLZ21} that the lower bound on this family of graphs is $\Omega\big(N^{\frac{1}{3}}\cdot T^{\frac{2}{3}}\big)$.
	
Despite the gap between current upper and lower bounds on specific instances, there seems to be some technical barrier for the algorithm design. Almost all current algorithms for bandit with graph feedback in adversarial setting are variants of \emph{online stochastic mirror descent} (OSMD). The choice of the potential function is key to an optimal algorithm and relies on the feedback structure. An empirical fact is that, if the feedback graph is sparse (e.g., MAB), Tsallis entropy is the optimal choice while for dense feedback graphs (e.g., learning with expert advice, or the complete bipartite graphs studied in~\cite{CHLZ21}), the negative entropy results in optimal regret. Is there a uniform treatment for all graphs, or in other words, can we interpolate between various potential functions?
	
We propose to answer the above question via first understanding the following instance: Suppose there are $m$ graphs $G_1,G_2,\dots,G_m$ and we know the optimal algorithm for them respectively. What is the optimal algorithm for $G\defeq\bigcup_{\bar k\in [m]} G_{\bar k}$\footnote{We prefer to use $\bar k$ as the index for subgraphs throughout the paper.}, which is the \emph{disjoint union} of these $m$ graphs. This model interpolates between MAB (let each $G_{\bar k}$ be two singleton vertices with self-loops) and full feedback graph (let $m=1$ and $G_1$ be the full feedback graph). 
	
In this article, we study a more general setting. Let $G=(V,E)$ and $V_1,\dots,V_m$ be a partition of $V$. For every $\bar k\in[m]$, let $G_{\bar k}=G[V_{\bar k}]$ be the subgraph of $G$ induced by $V_{\bar k}$. We design an algorithm for $G$ by viewing it as a graph made up of small graphs. To this end, we define the \emph{incidence graph} $H=(V_H,E_H)$ where $V_H=[m]$ and $(i,j)\in E_H$ iff there are some $(u,v)\in E$ with $u\in V_i$ and $v\in V_j$. Given any sequence of the loss vectors $\ell^{(1)},\dots,\ell^{(T)}$ in $T$ rounds, we can define the \emph{projection instance}, namely the instance with feedback graph $H$ (along with carefully designed ``projected'' loss vectors $L^{(1)},\dots,L^{(T)}$) and $m$ \emph{restriction instances}, namely the instances with feedback graph $G_{\bar k}$ for all $\bar k\in[m]$ (along with the restriction of $\ell^{(1)},\dots,\ell^{(T)}$ on $G_{\bar k}$).
	 
We propose a new algorithmic framework for solving the problem. We simultaneously maintain $m+1$ OSMD algorithms for the projection instance and all the restriction instances. In each round, we first choose a subgraph $G_{\bar k}$ for $\bar k\in [m]$ according to the information provided by the projection instance, and then pick the arm in $G_{\bar k}$ following the information provided by the restriction instance on $G_{\bar k}$. Surprisingly, the regret of this two-level OSMD can be nicely decomposed into the sum of regret of the projection instance and the regret of the restriction instance containing the optimal arm (plus some exploration penalties). An informal statement of our regret decomposition theorem is \Cref{thm:decomp-informal} below and its formal statement is \Cref{thm:general regret} in \Cref{sec:regret-decomp}.
	
\begin{theorem}[Regret Decomposition Theorem, informal]\label{thm:decomp-informal} 
	There exists an algorithm such that the regret $R_G(T)$ on $G$ against any loss vector $\ell^{(1)},\dots,\ell^{(T)}$ can be decomposed as
	\[
		R_G(T) \le R_H(T) + R_{G_{\bar k^*}}(T)+ [\mbox{exploration penalty for $H$}] + [\mbox{exploration penalty for $G_{\bar k^*}$}],
	\]
	where $G_{\bar k^*}$ is the subgraph containing the optimal arm.
\end{theorem}

	Our algorithm allows that the graphs $G_1,\dots,G_m$ are a mixture of strongly observable graphs and weakly observable graphs. Moreover, it allows to use different potential functions on the projection instance $H$ and on each restriction instance $G_{\bar k}$. This property is crucial to obtain optimal algorithms in a uniform way. 

	The regret decomposition theorem does not provide an explicit regret bound. For a specific instance, one needs to realize it with concrete potential functions and exploration rates. We therefore introduce some ways of the realizations of the regret decomposition theorem, depending on the partition and the graph structure. 
	
	A natural realization, with a heuristic on how to partition the graph, is described in \Cref{sec:realization}. The potential function we choose for the projection instance $H$ is a separable function $\Psi(\*y)=\sum_{\bar k\in [m]}\Psi_{\bar k}(\*y(\bar k))$ where if $\bar k$ is in the ``strong observable part'' (formally defined in \Cref{sec:decomp}) without self-loop, then $\Psi_{\bar k}$  is the negative entropy and otherwise $\Psi_{\bar k}$ is the Tsallis entropy. The potential functions we choose for restriction instances are negative entropies. This special realization results in a concrete upper bound stated in \Cref{thm:realized-upper-bound}, which is already better than previous algorithm on many instances. We then introduce a more sophisticated realization with \emph{adaptive} exploration. This realization outperforms the previous one on graphs with bounded degree and results in optimal regret in many cases. We also discuss the issue on how to find an optimal partition in general in \Cref{sec:realization}.

	We show that our new algorithmic framework accurately captures the regret of the bandit with graph feedback by introducing some applications of these realizations, . We first consider those $C$-corrupted strongly observable graphs. That is, the weakly observable graphs containing at most $C$ vertices that are not strongly observable. In~\cite{ACBDK15}, it was shown that as long as one vertex in a strongly observable graph becomes weakly observable (by removing the self-loop or an edge incident to it, say), the regret's dependency on $T$ suddenly changes from $T^{\frac{1}{2}}$ to $T^{\frac{2}{3}}$. However, it was not clear how the dependency on the graph $G$ is changed. We prove that
	
	\begin{theorem}\label{thm:c-corrupted intro}
	If $G$ is a weakly observable graph containing at most $C$ vertices which are not strongly observable, then for sufficiently large $T$, any loss sequence $\ell^{(1)},\dots,\ell^{(T)}$, the regret of our realization is at most $9\cdot\tp{4C}^{\frac{1}{3}}\cdot T^{\frac{2}{3}}$. 
	\end{theorem}
	The upper bound contains no term in $\abs{G}$ and is tight in terms of $C$. It can be explained by our decomposition theorem as follows: We can decompose the graph into (at least) two parts, one containing strongly observable vertices and the other one containing those $C$ corrupted vertices. The regret from the first part is $\tilde O\big(\alpha(G)\cdot T^{\frac{1}{2}}\big)$ and the regret from the second part is $O\big(C^{\frac{1}{3}}\cdot T^{\frac{2}{3}}\big)$. It would be clear from the bounds in \Cref{sec:realization} that the regret of $G$ is dominated by the sum of the two, and therefore dominated by $O\big(C^{\frac{1}{3}}\cdot T^{\frac{2}{3}}\big)$ for sufficiently large $T$. This also explains the phenomenon of ``abrupt change in regret'' on \emph{loopy stars} discussed in~\cite{ACBDK15} and improves results therein.
	
	We then consider the disjoint union of graphs mentioned before. Generally speaking, one can always plug previous OSMD algorithm for each disjoint subgraph into our two-level algorithmic framework and obtain improved algorithm for the whole graph. For example, we prove that
	
	\begin{theorem}\label{thm:union-of-cliques intro}
	If $G$ is the disjoint union of $m\ge 2$ loop-less cliques and the $\bar k^{\!{th}}$ clique is of size $n_{\bar k}$. Then the mini-max regret of $G$ satisfies 
	\[
	R_G^*(T) = O\Big(\Big(\sum_{\bar k=1}^m \log n_{\bar k}\Big)^{\frac{1}{3}}\cdot T^{\frac{2}{3}}\Big).
	\]
\end{theorem}
We further apply our algorithm to graphs of bounded degree and obtain optimal algorithms. This resolves an open problem in~\cite{CHLZ21} where they asked for the optimal algorithm for undirected cycles.

\begin{theorem}\label{thm:bounded-degree}
	If a directed weakly observable graph $G$ is of bounded in-degree with $N$ vertices, then for any sufficiently large $T>0$ and any loss vector $\ell^{(1)},\dots,\ell^{(T)}$, the regret is 
	$
	O\big(N^{\frac{1}{3}}\cdot T^{\frac{2}{3}}\big).
	$
\end{theorem}

Note that any weakly observable graph contains a subgraph of bounded in-degree and removing edges never decreases its mini-max regret. As a result, $O\big(N^{\frac{1}{3}}\cdot T^{\frac{2}{3}}\big)$ is a universal upper bound of regret for \emph{any} weakly observable graph. This improves previous best universal upper bound $O\big(\tp{N\log N}^{\frac{1}{3}}\cdot T^{\frac{2}{3}}\big)$ in~\cite{ACBDK15,CHLZ21}.

We also prove that for every graph of bounded out-degree, there exists some loss vectors yielding $\Omega\big(N^{\frac{1}{3}}\cdot T^{\frac{2}{3}}\big)$ regret. Therefore, the regret of a graph with bounded out-degree is $\Theta\big(N^{\frac{1}{3}}\cdot T^{\frac{2}{3}}\big)$. 

\begin{theorem}\label{thm:out-bounded-theta}
	Let $G$ be a weakly observable graph of bounded out-degree with $N$ vertices. Then for sufficiently large $T>0$, the mini-max regret satisfies
	\[
	R^*(T) = \Theta\Big(N^{\frac{1}{3}}\cdot T^{\frac{2}{3}}\Big).
	\]
\end{theorem}

%

\bigskip
\paragraph{\bf Related Works}
\emph{Multi-armed bandit}(MAB) is a classic and well-explored problem of sequential decision introduced in~\cite{RH52}. The work of~\cite{ACBFS02} proved that the mini-max regret of MAB is $\tilde\Theta(\sqrt{NT})$ in adversarial setting and~\cite{LTGA21} gives a tighter upper bound $\sqrt{2NT}$ which is the best known result so far. Another well-known problem is \emph{learning with expert advice} which was studied in~\cite{LNWM94},~\cite{VG90},~\cite{FS97}, etc. The regret of learning with expert advice model was proved to be $\Theta(\sqrt{T\log N})$ in~\cite{FS97}.  Widely used traditional algorithms for sequential decision problems include \emph{Thompson sampling}, \emph{upper confidence bound} (UCB) and EXP3. The algorithm \emph{Online stochastic mirror descent} (OSMD) was developed by~\cite{Nem79} and~\cite{NY83} which reaches the tight bound for both MAB and learning with expert advice by choosing appropriate potential functions. The work of~\cite{MS11} introduced a more general feedback model using a graph which allows the player to observe the out-neighbors of the chosen arm. Studies of this model includes those on fixed graphs (e.g.,~\cite{MS11},~\cite{ACBDK15},~\cite{CHLZ21}), time-varying graphs (e.g.,~\cite{KNVM14},~\cite{ACBDK15}) and random graphs (e.g.,~\cite{NCSYO17},~\cite{LBS18},~\cite{LCWL20}). The work of~\cite{ACBDK15} add an exploration term into standard OSMD which is defined by domination number and reaches an upper bound of $O((\delta \log N)^{\frac{1}{3}}T^{\frac{2}{3}})$ where $\delta$ is the weak domination number of the feedback graph. The work of~\cite{CHLZ21} further improved the result to $O((\delta^* \log N)^{\frac{1}{3}}T^{\frac{2}{3}})$ where $\delta^*$ is the fractional weak domination number of the feedback graph. 


\section{Preliminaries}
Let $n\in \^N$ be a positive integer. We use $[n]$ denote the set $\set{1,2,\dots,n}$. $\Delta_{n-1}=\set{\*x\in \^R_{\geq 0}: \sum_{i=1}^n \*x(i)=1}$ is the $n-1$ dimension probability simplex. Let $(\*e_i^{[n]})_{i=1}^n$ be the standard basis of $\^R^n$ which means  for every $j\in [n]$, $\*e_i^{[n]}(j)=1$ if $j=i$ and $0$ otherwise.
Let $\*1^{[n]}\in \^R^n$ be a vector that every element is $1$ or equivalently $\*1^{[n]} = \sum_{i=1}^n\*e_i^{[n]} $.

\subsection{Graphs}
Let $G=\tp{V,E}$ be a directed graph with possibly self-loops where $|V|=N$. When we say $G$ is undirected, we understand an undirected edge $\set{u,v}$ as two directed edges $\tp{u,v}$ and $\tp{v,u}$. For every $S\subseteq V$, we use $G[S]$ to denote the subgraph of $G$ induced by $S$. Let $m\in \^N$ be a positive integer. Let $\set{V_1,\dots,V_m}$ be a partition of $V$. Define the \emph{incidence graph} $H=(V_H,E_H)$ w.r.t the partition as $V_H=[m]$ and $E_H=\set{(i,j)\in [m]^2\cmid i\ne j\land\exists u\in V_i, v\in V_j, (u,v)\in E}$. For every $\bar k\in [m]$, we usually use $G_{\bar k}=(V_{\bar k},E_{\bar k})$ to denote $G[V_{\bar k}]$. We call each $V_{\bar k}$ a block of the partition. Once we view $G$ as an instance of bandit with graph feedback, we call $H$ the \emph{projection instance} and each $G_{\bar k}$ a \emph{restriction instance}.

%

For every $v\in V$, we define $\Nin(v) = \set{u\in V: (u,v)\in E}$ and $\Nout(v)=\set{u\in V: (v, u)\in E}$ as the set of in-neighbors and out-neighbors of $v$ respectively. Then we use $|\Nin(v)|$ and $\abs{\Nout(v)}$ to denote the in-degree and out-degree of $v$ respectively. A set $S\subseteq V$ is an independent set if there is no edge between any two vertices in $S$. A set with a self-loop vertex can not be an independent set. The notion of $t$-packing independent set $S$ in a graph $G=(V,E)$ is defined as an independent set $S\subseteq V$ satisfying for every $u\in V$, $\abs{\Nout(u)\cap S}\le t$.

We say a vertex $v\in V$ is \emph{non-observable} if $\Nin(v)=\emptyset$, otherwise, it is \emph{observable}. A graph with non-observable vertices is called a non-observable graph, otherwise, it is an observable graph. A vertex $v$ is called strongly observable if either $v$ has a self-loop or $\Nin(v)=V\setminus\set{v}$. A graph is a strongly observable graph if every vertex of it is strongly observable. Weakly observable vertices refer to vertices which are neither non-observable nor strongly observable. Graphs which are neither non-observable nor strongly observable are called weakly observable graphs. 

Consider the following linear programming $\+P$ defined on $G_{\bar{k}}$ for every $\bar k\in [m]$ such that $|V_{\bar k}|\geq 2$: 
\[
\-{minimize} \quad \sum_{v\in V_{\bar{k}}} x_v, \mbox{ s.t. } \sum_{v\in N_{in}(u)\cap V_{\bar{k}}} x_v\geq 1, \forall u\in V_{\bar{k}} \quad\mbox{and}\quad 0\leq x_v\leq 1, \forall v\in V_{\bar{k}}.
\]

%

We use $\delta^*_{\bar k}(G_{\bar k})$ to denote the optimum of $\+P$. We call $\delta^*_{\bar k}(G_{\bar k})$ the local fractional weak domination number of $G_{\bar k}$ and when $G_{\bar k}$ is clear from the context, we use $\delta^*_{\bar k}$ for briefty. We use $x^*_{\bar k,j}$ to denote the corresponding solution of $\+P$ for $j\in [n_{\bar{k}}]$. Let $\ol\delta^*=\sum_{\bar{k}\in[m]\colon \abs{V_{\ol k}}\geq 2} \delta^*_{\bar{k}}$. 
Note that $\ol\delta^*$ here is different from $\delta^* = \delta^*(G)$ in~\cite{CHLZ21} which is the (global) fractional domination number. 

\subsection{Bandit with Graph Feedback}
Let $G=\tp{V,E}$ be a directed graph and $V=[N]$ be the collection of bandit arms. Let $T\in \^N$ be the time horizon. The structure of $G$ and the value of $T$ is known by the player. Bandit with graph feedback, or graph bandit for short, is an online decision problem. The player design an algorithm $\+A$ such that in each round $t=1,2,\dots T$:
\begin{enumerate}
	\item The algorithm $\+A$ computes a distribution $X^{(t)}\in \Delta_{N-1}$ and chooses an arm $A_t\in[N]$ by sampling from $X^{(t)}$;
	\item The adversary chooses a loss function $\ell^{(t)}:[N]\rightarrow [0,1]$;
	\item The player pays $\ell^{(t)}(A_t)$ and observes $\ell^{(t)}(j)$ for $j\in \Nout(A_t)$.
\end{enumerate}

For a fixed loss function sequence $\+L=\set{\ell^{(1)}, \ell^{(2)},\dots,\ell^{(T)}}$, let the best arm $a^* =$ $\argmin_{a\in [N]}$ $\sum_{t=1}^T$ $\ell^{(t)}(a)$. We can view the loss function $\ell^{(t)}$ as a vector and $\ell^{(t)}(j)$ is the value at its $j^{\-{th}}$ coordinate. The regret of the algorithm with respect to a fixed arm $a\in [N]$ is defined by $R_a(G,T,\+A,\+L)=\E{\sum_{t=1}^T \ell^{(t)}(A_t)} - \sum_{t=1}^T \ell^{(t)}(a)$ and the expectation is with respect to the randomness of the algorithm. When the context is clear, we write the regret as $R_a(T)$ for briefty. Furthermore, if not otherwise specified, the regret we refer to is $R_{a^*}(T)$ which is shortened to $R(T)$. The purpose of the game is to design a best algorithm against the worst adversary, that is, to achieve the mini-max regret $R_G^*(T)=\inf_{\+A}\sup_{\+L} R_{a^*}(G,T,\+A,\+L)$. We sometime drop the subscript $G$ and write $R^*(T)$ if $G$ is clear from the context.

Recall the notion of $t$-packing independent set $S$ defined before. The following lower bound of the mini-max regret was proved in~\cite{CHLZ21}:

\begin{proposition}\label{prop:packing-lb}
	For any algorithm, any weakly observable graph containing a $t$-packing independent set $S$ suffers $\Omega\Big(\max\set{\log\abs{S},\frac{\abs{S}}{t}}^{\frac{1}{3}}\cdot T^{\frac{2}{3}}\Big)$ regret on some loss vector sequences. 	
\end{proposition}

\subsection{Optimization}

Let $V\in \^R^n$ be a convex set. For a convex function $F\colon\^R^n\rightarrow \^R\cup \set{\infty}$, the domain of $F$ is $\-{dom}(F)=\set{\*x\in \^R^n: F(\*x)<\infty}$. Assume $\-{dom}(F)$ is open and $F$ is differentiable in its domain. Given $\*x,\*y\in \-{dom}(F)$, the Bregman divergence with respect to $F$ is $B_{F}(\*x,\*y)=F(\*x)-F(\*y)-\nabla_{\*x-\*y}(\*y)$ where $\nabla_{\*v}(\*y)$ is the directional derivative of $F$ in direction $\*v$ at $\*y$. The diameter of $V$ with resepct to $F$ is $D_{F}(V)=\max_{\*x,\*y\in V}F(\*x)-F(\*y)$. Negative entropy refers to the function $\Phi\colon\^R^n_{\geq 0}\rightarrow \^R\cup  \set{\infty}$ that $\Phi(\*x)=\sum_{i=1}^n \*x(i)\log \*x(i)$. Given a constant $h\in(0,1)$, the Tsallis entropy $\Psi:\^R^n_{\geq 0}\rightarrow \^R\cup \infty$ with respect to $h$ is defined by $\Psi(\*x)=\sum_{j=i}^n -\*x(i)^{h}$. In this work, we take $h=\frac{1}{2}$.

Let $A\in \^R^n\times \^R^n$ be a semi-definite positive matrix and $\*x\in \^R^n$ be a column vector, the norm with respect to $A$ is defined by $\|\*x \|_{A}\defeq\sqrt{\*x^{\!T}A\*x}$. When $A=\nabla^2 \Psi$ is the Hessian matrix of some function $\Psi$, we use $\|\*x \|_{\nabla^{-2}\Psi}$ to denote  $\|\*x \|_{\tp{\nabla^2 \Psi}^{-1}}$.

\subsection{Online Stochastic Mirror Descent}
Given a convex potential function $\Psi$ and a convex set $\+X$, OSMD starts with a distribution $X^{(1)}=\arg\min_{\*x\in \+ X} \Psi(\*x)$. In every round $t\in [T]$, it plays $A_t\sim X^{(t)}$, pays corresponding loss and gains some observation of the arms. With a loss estimator $\hat\ell^{(t)}$ of the real loss vector $\ell^{(t)}$ and a uniform step size $\eta$, it updates by $X^{(t+1)}=\arg\min_{\*x\in \+ X}\eta\inner{\*x}{\hat \ell^{(t)}}+B_{\Psi}(\*x,X^{(t)}) $. 
\begin{proposition}\label{prop:osmd}
	The regret of OSMD satisfies that
	$	
		R_{a^*}(T)\leq \frac{D_{\Psi}(\+ X)}{\eta} + \frac{\eta}{2} \sum_{t=1}^T \sup_{\*y\in [\hat{X}^{(t)},X^{(t)}]}$ $ \|\hat\ell^{(t)} \|^2_{\grad^{-2}\Psi(\*y)},
	$
	where $\hat{X}^{(t)} = \argmin_{\*x\in \-{int}\tp{\-{dom(\Psi)}}} \eta\inner{\*x}{\hat \ell^{(t)}}+B_{\Psi}(\*x,X^{(t)})$.
\end{proposition}
More details on OSMD can be found in e.g.~\cite{ZL19}.

\section{Regret Decomposition}\label{sec:regret-decomp}

In this section, we describe our algorithm based on a graph partition and state the regret decomposition theorem. We first define the notion of \emph{legal partition}, the main data structure that our algorithm relies on in \Cref{sec:decomp} and present the algorithm in \Cref{sec:algo}. We also provide the analysis of the algorithm and the proof of the main theorem in \Cref{sec:analysis}.

\subsection{Legal Partition}\label{sec:decomp}

Let $G=(V,E)$ be a directed graph with possible self-loops. Let $V_1,V_2,\dots,V_m$ be a partition of $V$. Recall that for every $\bar k\in [m]$, we let $G_{\bar k}=(V_{\bar k},E_{\bar k})\defeq G[V_{\bar k}]$ be the subgraph of $G$ induced by $V_{\bar k}$ and let $n_{\bar k}=\abs{V_{\bar k}}$. For every $\bar k\in [m]$, we call $V_{\bar k}$ a block of the partition.

We say a partition $\set{V_1,V_2,\dots,V_m}$ of $V$ is \emph{legal (for our algorithm)} if every subgraph $G[V_{\bar k}]$ is observable and it can be further partitioned into two groups $U_1=\set{1,2,\dots,s}$ and $U_2=\set{s+1,s+2,\dots,m}$ satisfying
	\begin{itemize}
		\item $n_{\bar k}=1$ for all $\bar k\in U_1$ and $n_{\bar k}>1$ for all $\bar k\in U_2$;
		\item For every $\bar k\in U_1$, the vertex $v_{\bar k}$ in the singleton set $V_{\bar k}$ is strongly observable in $G$.
	\end{itemize}
Note that we allow $U_1=\emptyset$ or equivalently $s=0$. We call $U_1$ (when referring to an index), or sometimes $\bigcup_{\bar k\in U_1} V_{\bar k}$ (when referring to an arm), the \emph{strongly observable part} of the partition.

In fact, our algorithm will treat $G\left[\bigcup_{\bar k\in U_1}V_{\bar k}\right]$ as a strongly observable instance and treat each $G[V_{\bar k}]$ for $\bar k\in U_2$ as a weakly observable instance (even though it is not). The intuition behind the definition is that the strongly observable graphs are more friendly to the player comparing to weakly observable graphs in terms of the mini-max regret ($\Theta(T^{\frac{1}{2}})$ v.s. $\Theta(T^{\frac{2}{3}})$). Therefore, our algorithm can take this advantage when a weakly observable graph contains a large strongly observable subgraph. This is crucial to some of the optimal algorithms in \Cref{sec:application}. An example of a legal partition and its corresponding incidence graph is illustrated in Figure~\ref{fig:partition}.
\begin{figure}[h!]
	\centering
        \begin{minipage}[t]{0.5\textwidth}
            \centering
            \begin{tikzpicture}
                \tikzset{mynode/.style=fill, circle, inner sep=2pt, minimum size=3pt}
                \node[mynode] (selfloop1) at (0.3,10){};
                \node[mynode] (selfloop2) at (1.7,10){};
                \node[mynode] (selfloop3) at (3.1,10){};
                \node[mynode] (selfloop4) at (5.9,10){};
                \node[color=black] at (4.5,10){$\cdots$};
                \foreach \j in {1,2,3,4}{
                    \path[-stealth] (selfloop\j) edge [loop] (selfloop\j);
                }
                \draw[red,dashed] (0.3,10) circle (0.4);
                \draw[red,dashed] (1.7,10) circle (0.4);
                \draw[red,dashed] (3.1,10) circle (0.4);
                \draw[red,dashed] (5.9,10) circle (0.4);
                \node[color=red,font=\fontsize{6}{6}\selectfont] at (0.3,10.55){$G_1$};
                \node[color=red,font=\fontsize{6}{6}\selectfont] at (1.7,10.55){$G_2$};
                \node[color=red,font=\fontsize{6}{6}\selectfont] at (3.1,10.55){$G_3$};
                \node[color=red,font=\fontsize{6}{6}\selectfont] at (5.9,10.55){$G_s$};
                \draw[rounded corners, color=blue,dash pattern= on 10pt off 5pt, opacity=0.5] (-0.2, 10.7) rectangle (6.45,9.5);
                \node[color=blue,font=\fontsize{6}{6}\selectfont] at (3,10.9){Strongly Observable Part};

                \node[mynode] (bip11) at (0,9){};
                \node[mynode] (bip21) at (1,9){};
                \node[mynode] (bip12) at (0,8){};
                \node[mynode] (bip22) at (1,8){};
                \node[mynode] (bip13) at (0,6){};
                \node[mynode] (bip23) at (1,6){};
                \node[color=black] at (0,7){$\vdots$};
                \node[color=black] at (1,7){$\vdots$};
                \foreach \j in {1,2,3}{
                    \foreach \i in {1,2,3}{
                        \path[-stealth] (bip1\j) edge (bip2\i);
                        \path[-stealth] (bip2\i) edge (bip1\j);
                    }
                }
                \draw[rounded corners, color=red,dashed] (-0.2, 9.2) rectangle (1.2,5.8);
                \node[color=red,font=\fontsize{6}{6}\selectfont] at (0.5,5.6){$G_{s+1}$};

                \node[mynode] (starcenter) at (2.8,8.5){};
                \path[-stealth] (starcenter) edge [loop] (starcenter); 
                
                \foreach \j in {0,1,2,4}{
                        \pgfmathparse{cos(30*(\j+1)};
                        \tikzmath{\x =\pgfmathresult;}
                        \pgfmathparse{sin(30*(\j+1)};
                        \tikzmath{\y =\pgfmathresult;}
                        \node[mynode] (star\j) at (2.8-\x,8.5-\y*2){};
                }
                \pgfmathparse{cos(60)};
                \tikzmath{\x =\pgfmathresult;}
                \pgfmathparse{sin(60)};
                \tikzmath{\y =\pgfmathresult;}
                \node[color=black] at (2.8+1.2*\x,8.5-\y*2.2){\begin{rotate}{85}$\ddots$\end{rotate}};
                \foreach \j in {0,1,2,4}{
                    \path[-stealth] (starcenter) edge (star\j);
                }
                \draw[dashed, red, rotate around={90:(2.8,7.5)}] (2.8,7.5) ellipse (1.5 and 1);
                \node[color=red,font=\fontsize{6}{6}\selectfont] at (2.8,5.6){$G_{s+2}$};

                \node[color=black] at (4.5,7.5){$\cdots$};

                \foreach \j in {0,1,2,3,4,5,6}{
                    \pgfmathparse{cos(45*(\j+4)};
                    \tikzmath{\x =\pgfmathresult;}
                    \pgfmathparse{sin(45*(\j+4)};
                    \tikzmath{\y =\pgfmathresult;}
                    \node[mynode] (cli\j) at (5.7-0.6*\x,7.5-1.5*\y){};
                }
                \pgfmathparse{cos(45*3)};
                \tikzmath{\x =\pgfmathresult;}
                \pgfmathparse{sin(45*3)};
                \tikzmath{\y =\pgfmathresult;}
                \node[color=black] at (5.7-0.9*\x,7.5-1.4*\y){\begin{rotate}{107}$\ddots$\end{rotate}};
                \foreach \j in {1,2,3,4,5,6}{
                    \path[-stealth] (cli0) edge (cli\j);
                }
                \foreach \j in {0,2,3,4,5,6}{
                    \path[-stealth] (cli1) edge (cli\j);
                }
                \foreach \j in {1,0,3,4,5,6}{
                    \path[-stealth] (cli2) edge (cli\j);
                }
                \foreach \j in {1,2,0,4,5,6}{
                    \path[-stealth] (cli3) edge (cli\j);
                }
                \foreach \j in {1,2,3,0,5,6}{
                    \path[-stealth] (cli4) edge (cli\j);
                }
                \foreach \j in {1,2,3,0,4,6}{
                    \path[-stealth] (cli5) edge (cli\j);
                }
                \foreach \j in {1,2,3,0,5,4}{
                    \path[-stealth] (cli6) edge (cli\j);
                }
                \draw[dashed, red, rotate around={90:(5.7,7.5)}] (5.7,7.5) ellipse (1.7 and 0.75);
                \node[color=red,font=\fontsize{6}{6}\selectfont] at (5.7,5.6){$G_{m}$};

                \path[-stealth] (selfloop1) edge (starcenter);
                \path[-stealth] (selfloop3) edge (cli3);
                \path[-stealth] (selfloop3) edge (cli0);
                \path[-stealth] (selfloop4) edge (cli5);
                \path[-stealth] (selfloop4) edge (cli6);
                \path[-stealth] (cli4) edge (selfloop3);
                \path[-stealth] (bip11) edge (selfloop2);
                \path[-stealth] (bip22) edge (star0);
                \path[-stealth] (bip23) edge (star0);
                \path[-stealth] (bip23) edge (star2);
                \path[-stealth] (star4) edge (selfloop4);

            \end{tikzpicture}
        \end{minipage}
        \begin{minipage}[t]{0.4\textwidth}
            \centering
            \begin{tikzpicture}
                \tikzset{mynode/.style=draw, circle, inner sep=1pt, minimum size=20pt, outer sep=1pt}
                \node[mynode, dashed, color=red, text=red,font=\fontsize{6}{6}\selectfont] (s1) at (0,4){$G_1$};
                \node[mynode, dashed, color=red, text=red,font=\fontsize{6}{6}\selectfont] (s2) at (1,4){$G_2$};
                \node[mynode, dashed, color=red, text=red,font=\fontsize{6}{6}\selectfont] (s3) at (2,4){$G_3$};
                \node[color=black] at (3,4){$\cdots$};
                \node[mynode, dashed, color=red, text=red,font=\fontsize{6}{6}\selectfont] (s4) at (4,4){$G_s$};

                \node[mynode, dashed, color=red, text=red,font=\fontsize{6}{6}\selectfont] (w1) at (0,2){$G_{s+1}$};
                \node[mynode, dashed, color=red, text=red,font=\fontsize{6}{6}\selectfont] (w2) at (1.33,2){$G_{s+2}$};
                \node[color=black] at (2.67,2){$\cdots$};
                \node[mynode, dashed, color=red, text=red,font=\fontsize{6}{6}\selectfont] (w3) at (4,2){$G_{m}$};

                \node[mynode,white] (a) at (0,0){};

                \path[-stealth] (s1) edge [double] (w2);
                \path[-stealth] (w1) edge [double] (s2);
                \path[-stealth] (w1) edge [double] (w2);
                \path[stealth-stealth] (s3) edge [double] (w3);
                \path[-stealth] (s4) edge [double] (w3);
                \path[-stealth] (w2) edge [double] (s4);
            \end{tikzpicture}
        \end{minipage}
    \caption{An example of a legal partition and its incidence graph}\label{fig:partition}
\end{figure}
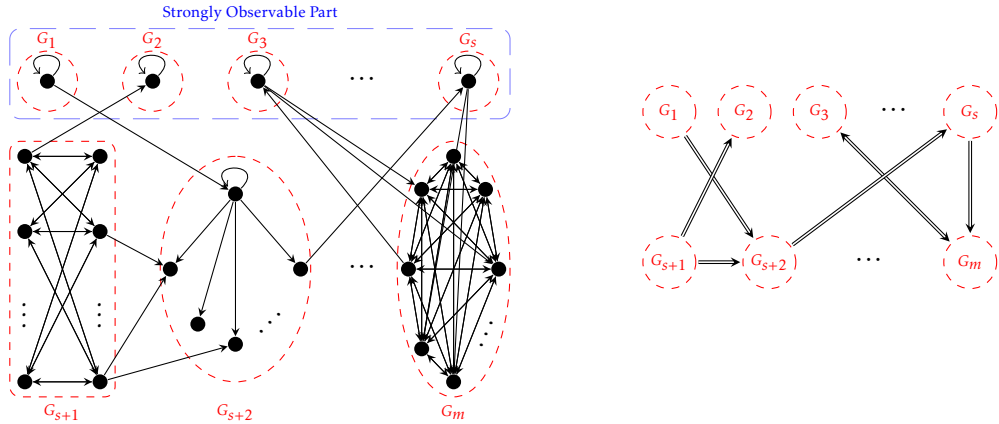

\subsection{The Algorithm}\label{sec:algo}

We assume settings in \Cref{sec:decomp}. That is, given a directed graph $G=(V,E)$, we fix a legal partition $V_1,V_2,\dots,V_m$ with $U_1$ and $U_2$. Each arm in $G$ is denoted by a pair $(\bar k,j)$ for $\bar k\in [m]$ and $j\in [n_{\bar k}]$. We further divide $U_1$ into $U_1^S$ and $U_1^{\bar S}$ where $U_1^S\subseteq U_1$ is the indices of those singleton sets containing an arm with a self-loop and $U_1^{\ol S}=U_1\setminus U_1^S$.

Speaking at a very high level, our algorithm is a two-level online stochastic mirror descent algorithm: We first pick a block $\bar k\in [m]$, and then pick an arm in $V_{\bar k}$. Therefore, in each round $t\in [T]$, we maintain two families of probability distributions:
\begin{itemize}
	\item We first maintain a distribution $Y^{(t)}\in \Delta_{m-1}$ on all $m$ blocks;
	\item For every $\bar k\in [m]$, we maintain a distribution $X_{\bar k}^{(t)}\in \Delta_{n_{\bar k}-1}$.
\end{itemize}

Since blocks in $U_1$ only contain one arm, for every $\bar k\in U_1$,  $X^{(t)}_{\bar k}$ is a distribution on a singleton. As a result, those arms belong to  $\bigcup_{\bar k\in U_1}V_{\bar k}$ are essentially explored by the rule $Y^{(t)}$. We introduce a convex potential function $\Psi:\^R^{m}\to\^R$ for $Y^{(t)}$. 

Those arms in $V_{\bar k}$ with $\bar k \in U_2$ are explored in a two-stage manner. For every such $X_{\bar k}^{(t)}$, we introduce a convex potential function $\Phi_{\bar k}:\^R^{n_{\bar k}}\to \^R$. 

We also define some exploration terms, locally and globally, as follows:
\begin{itemize}
	\item We define the \emph{global exploration factor}, denoted by $\gamma^{(t)}(\cdot)$, over all arms in $V$. That is, $\gamma^{(t)}:(\bar k,j)\mapsto \gamma^{(t)}((\bar k,j))\in [0,1]$ assigns each arm some chance to be explored at the \emph{first stage}.  Let $\ol\gamma^{(t)}\defeq \sum_{\bar k\in [m]}\sum_{j\in [n_{\bar k}]} \gamma^{(t)}((\bar k,j))$ be the total global exploration rate.
	\item For every block $\bar k\in U_2$, we define the \emph{local exploration factor} in $V_{\bar k}$, denoted by $\gamma^{(t)}_{\bar k}(\cdot)$, over all arms in $V_{\bar k}$. Similarly, $\gamma^{(t)}_{\bar k}: j\mapsto \gamma^{(t)}_{\bar k}(j)\in [0,1]$ assigns each arm in $V_{\bar k}$ some chance to be explored at the \emph{second stage}. We also let $\ol\gamma^{(t)}_{\bar k}\defeq \sum_{j\in [n_{\bar k}]}\gamma^{(t)}_{\bar k}(j)$ be the total local exploration rate in $V_{\bar k}$.
\end{itemize}

Assuming notations above, the implementation details can be found in \Cref{alg:osmd-decomp}.
Assume $Y^{(1)}$ and $X^{(1)}_{\bar k}$ for all $\bar k\in [m]$ are well initialized. In each round $t=1,2,\dots, T$, the behavior of the player includes:
\begin{itemize}
    \item Sampling:
    \begin{itemize}
        \item For each block $\bar k\in U_2$, we take into account the local exploration factor and define
	\[
		\tilde X_{\bar k}^{(t)} = (1-\ol\gamma^{(t)}_{\bar k})\cdot X_{\bar k}^{(t)} + \gamma^{(t)}_{\bar k}.
	\]
	    \item For those arms $(\bar k,j)\in 	\bigcup_{\bar k\in U_2}V_{\bar k}$, we take into account the global exploration factor and play it with probability
	\[
		Z^{(t)}((\bar k,j)) = (1-\ol\gamma^{(t)})\cdot Y^{(t)}(\bar k)\cdot \tilde X^{(t)}_{\bar k}(j) + \gamma^{(t)}((\bar k,j)).
	\]
	    \item For those arms $(\bar k,j)\in 	\bigcup_{\bar k\in U_1}V_{\bar k}$, we play it with probability
	\[
		Z^{(t)}((\bar k,j)) = (1-\ol\gamma^{(t)})\cdot Y^{(t)}(\bar k)\cdot X^{(t)}_{\bar k}(j) + \gamma^{(t)}((\bar k,j)).
	\]
    \end{itemize}
    \item Observing:
    \begin{itemize}
        \item For every $(\bar k,j)\in \Nout(A_t)$ where $A_t$ is the chosen arm, observe $\ell^{(t)}((\bar k,j))$.
        \item For every $(\bar k,j)\in V$, define the unbiased loss estimator $\hat\ell^{(t)}_{\bar k}(j)$ (see \Cref{alg:line:omd-X} of \Cref{alg:osmd-decomp}).
        \item Define the loss of the block $\wh L^{(t)}(\bar k)$ for all $\ol k\in[m]$ (see \Cref{alg:line:L-1} and \Cref{alg:line:L-2} of \Cref{alg:osmd-decomp}).
    \end{itemize}
	\item Updating:
    \begin{itemize}
        \item For every $\bar k$, we update $X^{t+1}_{\bar k}$ using OSMD with $\hat\ell^{(t)}_{\bar k}$ and potential function $\Phi_{\bar k}$:
	\[
	X^{(t+1)}_{\bar k}=\argmin_{\*x\in\Delta_{n_{\bar k}-1}} \eta_{\bar k}\cdot \inner{\*x}{\hat\ell^{(t)}_{\bar k}}+B_{\Phi_{\bar k}}(\*x,X^{(t)}_{\bar k}),
	\]
	where $\eta_{\bar k}$ is the step size to be set.

	\item Update $Y^{(t)}$ with $\wh L^{(t)}$ and the potential function $\Psi$: 
	\[
	Y^{(t+1)}=\argmin_{\*y\in\Delta_{m-1}} \inner{\*y}{\wh L^{(t)} - c^{(t)}\cdot\*1^{[m]}}+B_{\Psi}(\*y,Y^{(t)}),
	\]
	where $c^{(t)}$ is a constant defined in \Cref{alg:line:Lconst}.
    \end{itemize}
	
\end{itemize}

We remark that the value of $\wh L^{(t)}(\bar k)$ is the expectation of $\hat\ell^{(t)}_{\bar k}(j)$ under the distribution $\tilde X^{(t)}_{\bar k}$ over $j\in [n_{\bar k}]$. It would be clear from the analysis that this choice is the key to make everything work. 

\begin{algorithm}
\TitleOfAlgo{Online StochasticMirror Descent for Composite Graphs}
	\Input{A feedback graph $G=(V,E)$ and a legal partition $\set{V_{\bar k}}_{\bar k\in [m]}$; sets of indices $U_1=U_1^S\cup U_1^{\ol{S}}$, $U_2$.}
	\BlankLine
	\Begin{
		\For {$\bar k\in U_2$} {
			$X^{(1)}_{\bar k}\gets \argmin_{\*x\in\Delta_{n_{\bar k}-1}}\Phi_{\bar k}(\*x)$\; 
		}
		\For {$\bar k\in U_1$} {
			$X^{(1)}_{\bar k}\gets 1$\;
		}
		$Y^{(1)}\gets \argmin_{\*y\in\Delta_{m-1}}\Psi(\*y)$\;
	\For {$t=1,2,\dots,T$} {
		\For {$\bar k\in U_2$}{
			Define the vector $\tilde X^{(t)}_{\bar k}$ as $\tilde X^{(t)}_{\bar k}(j) = (1-\ol\gamma^{(t)}_{\bar k})\cdot X^{(t)}_{\bar k}(j)+\gamma^{(t)}_{\bar k}(j)$\;
			Define the vector $Z^{(t)}$ as $Z^{(t)}((\bar k,j))\gets (1-\ol\gamma^{(t)})\cdot Y^{(t)}(\bar k)\cdot \tilde X^{(t)}_{\bar k}(j)+\gamma^{(t)}((\bar k,j))$\;
		}
		\For {$\bar k\in U_1$}{
			Define the vector $Z^{(t)}$ as $Z^{(t)}((\bar k,j))\gets (1-\ol\gamma^{(t)})\cdot Y^{(t)}(\bar k)\cdot X^{(t)}_{\bar k}(j) + \gamma^{(t)}((\bar k,j))$\;
		}
		Play the arm $A^{(t)}\sim Z^{(t)}$ and observe $\ell^{(t)}((\bar k,j))$ for all $(\bar k,j)\in \Nout(A_t)$\;
		\For {$\bar k\in [m]$ and $j\in [n_{\bar k}]$}{ 
			$\hat\ell^{(t)}_{\bar k}(j)\gets \frac{\*1[(\bar k,j)\in \Nout(A_t)]}{\sum_{a\in \Nin((\bar k,j))}Z^{(t)}(a)}\cdot \ell^{(t)}((\bar k,j))$\;
		}
		\For {$\bar k\in U_2$} {
			$\wh L^{(t)}(\bar k) = \sum_{j\in [n_{\bar k}]}\tilde X^{(t)}_{\bar k}(j)\cdot \hat\ell^{(t)}_{\bar k}(j)$\; \label{alg:line:L-1}
		}
		\For {$\bar k\in U_1$} {
			$\wh L^{(t)}(\bar k) = \hat\ell^{(t)}_{\bar k}(1)$\; \label{alg:line:L-2}
		}
		\For {$\bar k\in U_2$} {
			$X^{(t+1)}_{\bar k}\gets \argmin_{\*x\in\Delta_{n_{\bar k}-1}} \eta_{\bar k}\cdot \inner{\*x}{\hat\ell^{(t)}_{\bar k}}+B_{\Phi_{\bar k}}(\*x,X^{(t)}_{\bar k})$\;\label{alg:line:omd-X}
		}
		$c^{(t)}\gets \sum_{\bar k\in U_1^{\ol S}} \wh L^{(t)}(\bar k) \cdot Y^{(t)}(\bar k)$\; \label{alg:line:Lconst} \tcc{We shift  $\wh L^{(t)}$ by $c^{(t)}\cdot\*1^{[m]}$ to reduce its variance}
		$Y^{(t+1)}\gets \argmin_{\*y\in\Delta_{m-1}} \inner{\*y}{\wh L^{(t)} - c^{(t)}\cdot \*1^{[m]}}+B_{\Psi}(\*y,Y^{(t)})$\;
		\tcc{We hide the choice of ``learning rate'' in $\Psi$}

		}
	}
	\caption{Online Stochastic Mirror Descent for Composite Graphs}
	\label{alg:osmd-decomp}
\end{algorithm}

\subsection{Regret Decomposition Theorem}\label{sec:analysis}

The main result of this section is the following regret decomposition theorem.

Assume notations in \Cref{sec:algo}. We let $(\wh L^{(t)})'\defeq \wh L^{(t)} - c^{(t)}\cdot \*1^{[m]}$ where $c^{(t)}=\sum_{\ol k\in U_1^{\ol S}} \wh L^{(t)}(\ol k) \cdot Y^{(t)}(\ol k)$ is defined in \Cref{alg:line:Lconst} of \Cref{alg:osmd-decomp}.

\begin{theorem}[Regret Decomposition Theorem] \label{thm:general regret}
	Let $(\bar k^*,j^*)$ be a fixed arm. If $\bar k^*\in U_2$, then the regret of \Cref{alg:osmd-decomp} with respect to $(\bar k^*,j^*)$ is
	\begin{align*}
		&\quad R_{(\bar k^*,j^*)}(T) \le \Bigg\{D_{\Psi}(\Delta_{m-1}) + \frac{1}{2}\sum_{t=1}^T\E{\sup_{\*y\in [W^{(t)},Y^{(t)}]}\|(\wh L^{(t)})'\|_{\nabla^{-2}\Psi(\*y)}}\Bigg\}
		+ \sum_{t=1}^T\sum_{\bar k\in [m]}\sum_{j\in [n_{\bar k}]} \gamma^{(t)}((\bar k,j))\\
		&+\Bigg\{\frac{D_{\Phi_{\bar k^*}}(\Delta_{n_{\bar k^*-1}})}{\eta_{\bar k^*}}+\frac{\eta_{\bar k^*}}{2}\cdot\sum_{t=1}^T \E{\sup_{\*x\in [Q^{(t)}_{\bar k^*},X^{(t)}_{\bar k^*}]} \|\hat\ell^{(t)}_{\bar k^*}\|_{\grad^{-2}\Phi_{\bar k^*}(\*x)}}\Bigg\}
		+ \sum_{t=1}^T\sum_{j\in [n_{\bar k^*}]}\gamma^{(t)}_{\bar k^*}(j);
	\end{align*}
	and if $(\bar k^*,j^*)\in U_1$, then the regret of \Cref{alg:osmd-decomp} with respect to $(\bar k^*,j^*)$ is
	\begin{align*}
		R_{(\bar k^*,j^*)}(T) 
		&\le D_{\Psi}(\Delta_{m-1}) + \sum_{t=1}^T \tp{\frac{1}{2}\E{\sup_{\*y\in [W^{(t)},Y^{(t)}]}\|(\wh L^{(t)})'\|_{\nabla^{-2}\Psi(\*y)}}
		+ \sum_{\bar k\in [m]}\sum_{j\in [n_{\bar k}]} \gamma^{(t)}((\bar k,j))},
	\end{align*}
	where  $W^{(t)} = \arg\min_{\*w\in \-{int}(\-{dom}(\Psi))} \inner{\*w}{(\wh{L}^{(t)})'} + B_\Psi(\*w, Y^{(t)})$ and\\ $Q^{(t)}_{\bar{k}^*}=\arg\min_{\*q\in \-{int}(\-{dom}(\Phi_{\bar{k}^*}))}$ $\eta_{\bar{k}^*} \cdot \inner{\*q}{\hat{l}^{(t)}_{\bar{k}^*}} + B_{\Phi_{\bar{k}^*}}(\*q, X^{(t)}_{\bar{k}^*})$.
\end{theorem}

The regret decomposition theorem essentially says that the regret of the whole instance comes from four parts: the regret of the projection instance, the regret of the restriction instance, the cost of global exploration and the cost of local exploration.

The remain of this section outlines a proof of the theorem. The complete proof is in \Cref{sec:proof}.


Let us fix an arm $a^*=(\bar k^*,j^*)$. To ease the presentation, for every $t=1,2,\dots, T$, we define an $N$-dimensional vector $\hat\ell^{(t)}$ indexed by $(\bar k,j)$ pairs for every $\bar k\in [m], j\in [n_{\bar k}]$ satisfying $\hat \ell^{(t)}((\bar k,j)) = \hat\ell^{(t)}_{\bar k}(j)$. Clearly $\E{\hat\ell^{(t)}} = \ell^{(t)}$.

\begin{lemma}	\label{lem:regret}
	The regret of \Cref{alg:osmd-decomp} with respect to $(\bar k^*,j^*)$ is
	\begin{align*}
	R_{(\bar k^*,j^*)}(T)\le \sum_{t=1}^T\E{\inner{\wh L^{(t)}}{Y^{(t)}-\*e^{[m]}_{\bar k^*}}+\sum_{\bar k\in [m]}\sum_{j\in [n_{\bar k}]}\gamma^{(t)}((\bar k,j)) + \tp{\inner{\hat\ell^{(t)}_{\bar k^*}}{X^{(t)}_{\bar k^*}-\*e^{[n_{\bar k^*}]}_{j^*}}+\sum_{j\in [n_{\bar k^*}]}\gamma^{(t)}_{\bar k^*}(j)}\cdot\*1[\bar k^*\in U_2]}.
	\end{align*}
\end{lemma}


The key to prove Lemma~\ref{lem:regret} is to decompose the regret $\E{\inner{\hat\ell^{(t)}}{Z^{(t)}-\*e^{[N]}_{a^*}}}$ with appropriate choices of loss functions defined for the projection instance and restriction instances. By the definition of $Z^{(t)}$, we can verify that
\begin{align*}
    \E{\inner{\hat\ell^{(t)}}{Z^{(t)}-\*e^{[N]}_{a^*}}}&\le\E{\sum_{\bar k\in U_2}Y^{(t)}(\bar k)\sum_{j\in[n_{\bar k}]}\tilde X^{(t)}_{\bar k}(j)\cdot \hat\ell^{(t)}((\bar k,j)) + \sum_{\bar k\in U_1}Y^{(t)}(\bar k) \cdot \hat\ell^{(t)}((\bar k,1))-\hat\ell^{(t)}(a^*)}+\sum_{\bar k\in [m]}\sum_{j\in[n_{\bar k}]}\gamma^{(t)}((\bar k,j)).
\end{align*}
Recall that we let $\wh L^{(t)}(\bar k) = \sum_{j\in [n_{\bar k}]}\tilde X^{(t)}_{\bar k}(j)\cdot \hat\ell^{(t)}_{\bar k}(j)$\ for $\ol k \in U_2$ and $\wh L^{(t)}(\bar k) = \hat\ell^{(t)}_{\bar k}(1)$ for $\ol k \in U_1$. We can then write
\begin{align*}
  \E{\sum_{\bar k\in U_2}Y^{(t)}(\bar k)\sum_{j\in[n_{\bar k}]}\tilde X^{(t)}_{\bar k}(j)\cdot \hat\ell^{(t)}((\bar k,j)) + \sum_{\bar k\in U_1}Y^{(t)}(\bar k) \cdot \hat\ell^{(t)}((\bar k,1))-\hat\ell^{(t)}(a^*)}=\E{\inner{\wh L^{(t)}}{Y^{(t)}-\*e^{[m]}_{\bar k^*}}}+\E{\inner{\wh L^{(t)}}{\*e^{[m]}_{\bar k^*}}-\inner{\hat\ell^{(t)}_{\bar k^*}}{\*e^{[n_{\bar k^*}]}_{j^*}}}.
\end{align*} 
Finally by observing that if $\ol k^*\in U_2$,
\begin{align*}
    \E{\inner{\wh L^{(t)}}{\*e^{[m]}_{\bar k^*}}} &= \E{\sum_{j\in [n_{\bar k^*}]} \tilde X^{(t)}_{\bar k^*}(j)\cdot \hat\ell^{(t)}_{\bar k}(j)}\le \E{\inner{\hat\ell^{(t)}_{\bar k^*}}{X^{(t)}_{\bar k^*}}} + \sum_{j\in [n_{\bar k^*}]}\gamma^{(t)}_{\bar k^*}(j),
\end{align*}
and if $\bar k\in U_1$,
\[
\E{\inner{\wh L^{(t)}}{\*e^{[m]}_{\bar k^*}}} = \E{\wh L^{(t)}(\bar k^*)} = \E{\hat\ell^{(t)}_{\bar k^*}(1)}=\E{\inner{\hat\ell^{(t)}_{\bar k^*}}{\*e^{[n_{\bar k^*}]}_{\bar j^*}}}.
\]
See \Cref{proof:lem regret} for details of the calculation.

\bigskip
We then bound the regrets contributed by the projection instance and the restriction instance appeared in \Cref{lem:regret}. They are treated in Lemma~\ref{lem:regretY} and Lemma~\ref{lem:regretX} respectively. Both lemmas are consequences of Proposition~\ref{prop:osmd} via setting appropriate parameters. The details can be found in \Cref{proof:lem regretY} and \Cref{proof:lem regretX} respectively.

\begin{lemma}\label{lem:regretY}
    It holds that	
    \begin{align*}
        \sum_{t=1}^T \E{\inner{\wh L^{(t)}}{Y^{(t)}-\*e^{[m]}_{\bar k^*}}}&\le D_{\Psi}(\Delta_{m-1}) + \frac{1}{2}\sum_{t=1}^T\E{\sup_{\*y\in [W^{(t)}, Y^{(t)}]}\tp{\|(\wh L^{(t)})'\|_{\grad^{-2}\Psi(\*y)} }},
    \end{align*}
        where $W^{(t)} = \arg\min_{\*w\in \-{int}(\-{dom}(\Psi))} \inner{\*w}{(\wh{L}^{(t)})'} + B_\Psi(\*w, Y^{(t)})$. 
\end{lemma}

\begin{lemma}\label{lem:regretX}
	If $\bar k^*\in U_2$,
\[
	\sum_{t=1}^T\E{\inner{\hat\ell^{(t)}_{\bar k^*}}{X^{(t)}_{\bar k^*}-\*e^{[n_{\bar k^*}]}_{j^*}}} \le \frac{D_{\Phi_{\bar k^*}(\Delta_{n_{\bar k^*-1}})}}{\eta_{\bar k^*}}+\frac{\eta_{\bar k^*}}{2}\cdot\sum_{t=1}^T \E{\sup_{\*x\in [Q^{(t)}_{\bar{k}^*}, X^{(t)}_{\bar k^*}]} \|\hat\ell^{(t)}_{\bar k^*}\|_{\grad^{-2}\Phi_{\bar k^*}(\*x)}},
\]
where $Q^{(t)}_{\bar{k}^*}=\arg\min_{\*q\in \-{int}(\-{dom}(\Phi_{\bar{k}^*}))} \eta_{\bar{k}^*}\cdot \inner{\*q}{\hat{l}^{(t)}_{\bar{k}^*}} + B_{\Phi_{\bar{b}^*}}(\*q, X^{(t)}_{\bar{k}^*})$.
\end{lemma}

\section{Realization of the Regret Decomposition Theorem}\label{sec:realization}

The regret upper bound stated in \Cref{thm:general regret} relies on a given legal partition, the choices of potential functions and the value of various parameters (e.g., those ``exploration rates'' and ``learning rates''). In this section, we introduce two different realizations, depending on the graph structure and yielding improved and optimal regret bound in various settings. At last, we discuss the issue of ``optimal realization''.

\subsection{Realization for Well-Clustered Graphs}\label{sec:first-realization}


Motivated by the case when $G$ consists of disjoint union of subgraphs, we make the following heuristic assumption on a good legal partition for graphs that can be partitioned into well-clustered parts.

\begin{enumerate}
	\item It isolates a large ``strongly observable part'' from the graph, since the strongly observable graphs have small mini-max regret in general;
	\item Each of the remaining blocks is dense, so we can choose ``dense graph friendly'' potential functions to obtain small regret on restriction instances;
	\item The incidence graph is sparse, so we can choose a ``sparse graph friendly'' potential function to obtain small regret on the projection instance.
\end{enumerate}

We will see in \Cref{sec:application} that the rule of partition can yield improved regret when $G$ is the disjoint union of loop-less cliques and we make a heuristic step to assume that the rule generalizes to other graphs of similar structure. Our choice for potential functions is then clear: We let the potential function $\Psi$ for the projection instance be a separable one ($\Psi(\*y) = \sum_{\bar k\in [m]}\Psi_{\bar k}(\*y(\bar k))$), and each $\Psi_{\bar k}$ and $\Phi_{\bar k}$ is chosen in the following way.
\begin{enumerate}
	\item For a block $V_{\bar k}$ in the ``strongly observable part'', if it contains a self-loop, we let $\Psi_{\bar k}$ be \emph{Tsallis entropy}.
	\item For a block $V_{\bar k}$ in the ``strongly observable part'', if it does not contain a self-loop, we let $\Psi_{\bar k}$ be \emph{negative entropy}.
	\item For a block $V_{\bar k}$ not in the ``strongly observable part'', we let $\Psi_{\bar k}$ be \emph{Tsallis entropy}.
	\item For each restriction instance $V_{\bar k}$, we let $\Phi_{\bar k}$ be \emph{negative entropy}.
\end{enumerate}

We give a complete characterization of the regret bounds of this realization.
%

\begin{theorem}\label{thm:realized-upper-bound}
	Let $G=(V,E)$ be a directed graph instance. Let $V_1,V_2,\dots,V_m$ be a legal partition of $V$ with $U_1$ and $U_2$.  Let $U_1^{S}\subseteq U_1$ be the indices of those singleton sets containing an arm with a self-loop and $U_1^{\ol S}=U_1\setminus U_1^S$. Then for sufficiently large $T>0$, any loss sequence $\ell^{(1)},\dots,\ell^{(T)}$ and any arm $a^*=(\bar k^*,j^*)$ in $V$, the regret of \Cref{alg:osmd-decomp} with respect to $a^*$ satisifies
  \begin{equation*}
		R_{(\bar k^*,j^*)}(T)\leq \left\{
		\begin{aligned}
			&2\sqrt{2|U_1^S|}T^{\frac{1}{2}},& U_2=\emptyset {\ and\ } U_1^{\ol S}= \emptyset; \\
			&4\sqrt{6|U_1^S|}T^{\frac{1}{2}} + 2\sqrt{10\log(|U_1^{\ol S}|)} T^{\frac{1}{2}}+ T^{\frac{1}{2}},& U_2=\emptyset{\ and\ } U_1^{\ol S}\neq \emptyset; \\ 
			&3\cdot 2^{\frac{2}{3}}\tp{|U_2|\sum_{\bar{k}\in U_2} (\delta^*_{\bar{k}})^2}^{\frac{1}{6}} T^{\frac{2}{3}} + \frac{3}{2^{\frac{1}{3}}} \cdot \tp{\sum_{\bar{k}\in U_2} \delta_{\bar{k}}^* \log{n_{\bar{k}}}}^{\frac{1}{3}} T^{\frac{2}{3}}& \\
			&\quad\quad + 4\sqrt{|U_1^S|}T^{\frac{1}{2}},& U_2\neq \emptyset {\ and\ } U_1^{\ol S}= \emptyset; \\ 
			&6\cdot 2^{\frac{1}{3}}\tp{|U_2|\sum_{\bar{k}\in U_2} (\delta^*_{\bar{k}})^2}^{\frac{1}{6}} T^{\frac{2}{3}} + \frac{3}{2^{\frac{1}{3}}} \cdot \tp{\sum_{\bar{k}\in U_2} \delta_{\bar{k}}^* \log{n_{\bar{k}}}}^{\frac{1}{3}} T^{\frac{2}{3}} + \frac{\sqrt{6}}{3}T^{\frac{1}{2}}&  \\
			&\quad\quad + 4\sqrt{6|U_1^S|}T^{\frac{1}{2}}+2\sqrt{10\log(|U_1^{\ol S}|)} T^{\frac{1}{2}} + \frac{4T^{\frac{1}{3}} |U_2|^{\frac{5}{6}}}{2^{\frac{1}{3}}\tp{\sum_{\bar{k}\in U_2} (\delta^*_{\bar{k}})^2}^{\frac{1}{6}}},& U_2\neq \emptyset {\ and\ } U_1^{\ol S}\neq \emptyset. 
		\end{aligned}
		\right.
	\end{equation*}
\end{theorem}

\Cref{thm:realized-upper-bound} is proved in the following way. We realize the regret of the projection instance in \Cref{sec:regret-project} and the regret of restriction instances in \Cref{sec:regret-restrict} by picking appropriate parameters respectively. Equipped with these two lemmas, we apply \Cref{thm:general regret} on various types of partitions. The full proof of \Cref{thm:realized-upper-bound} is in \Cref{sec:proof-main}.

\subsubsection{Regret of the Projection Instance}\label{sec:regret-project}
In this section, we bound the regret contributed by the ``projection instance'', namely the term $\sum_{t=1}^T \E{\inner{\wh L^{(t)}}{Y^{(t)}-\*e^{[m]}_{\bar k^*}}}$. Remember that we delay the choice of step sizes for the projection instance here. In fact, we choose the potential function $\Psi(\*y)$ as a separable function so that it is Tsallis entropy on blocks indexed by $U_1^S$ and $U_2$ (with different learning rate), and it is negative entropy on blocks indexed by $U_1^{\ol S}$.
\begin{lemma}\label{lem:realized-regretY}
	Let $\Psi(\*y) = \sum_{\bar k\in U_2} \frac{-\sqrt{\*y(\bar k)}}{\eta} + \sum_{\bar k\in U_1^S} \frac{-\sqrt{\*y(\bar k)}}{\eta_S} + \sum_{\bar k\in U_1^{\ol S}} \frac{\*y(\bar k)\log(\*y(\bar k))}{\eta_{\ol S}}$ where $\eta$, $\eta_S$ and $\eta_{\bar S}$ are constants such that $\min_{i\in[m]}(\hat L^{(t)})'(i)\cdot \max\set{\eta, \eta_S, \eta_{\bar S}}\geq -\frac{1}{4}$ for every $t\in[T]$. Choose $\gamma_{\ol k}^{(t)}(j)=\frac{x^*_{\ol k, j}}{\delta^*_{\ol k}}\alpha$ for any $t\in[T]$, $\ol k\in U_2$ and $j\in [n_{\ol k}]$.
	\begin{itemize}
\item If $U_1^{\ol S}\neq \emptyset$ we have
	\begin{align*}
		D_{\Psi}(\Delta_{m-1}) + \frac{1}{2}\sum_{t=1}^T\E{\sup_{\*y\in [W^{(t)}, Y^{(t)}]}\tp{\|(\wh L^{(t)})'\|_{\grad^{-2}\Psi(\*y)} }}
		&\le \frac{\sqrt{|U_1^S|}}{\eta_S} + \frac{\log(|U_1^{\ol S}|+1)}{\eta_{\ol S}}
		+ \frac{\sqrt{|U_2|}}{\eta} +  16\eta T \frac{\sqrt{\sum_{\bar k\in U_2}(\delta^*_{\bar k})^2}}{\alpha} \\
		& + 8\tp{1+\frac{1}{1-\bar {\gamma}}}\eta_S T \sqrt{|U_1^S|}+ {2}\eta_{\ol S}T  + 8\eta T \sqrt{|U_2 |}.
	\end{align*}
		
	\item If $U_1^{\ol S}= \emptyset$ we have
	\begin{align*}
	 D_{\Psi}(\Delta_{m-1}) + \frac{1}{2}\sum_{t=1}^T\E{\sup_{\*y\in [W^{(t)}, Y^{(t)}]}\tp{\|(\wh L^{(t)})'\|_{\grad^{-2}\Psi(\*y)} }}
		\le\frac{\sqrt{|U_2|}}{\eta} + \frac{\sqrt{|U_1^S|}}{\eta_S}+  \eta T \frac{4\sqrt{\sum_{\bar k\in U_2}(\delta^*_{\bar k})^2}}{\alpha} + \frac{2}{1-\bar {\gamma}}\eta_S T \sqrt{|U_1^S|}.
	\end{align*}
	\end{itemize}
\end{lemma}

The key to prove Lemma~\ref{lem:realized-regretY} is to give an upper bound to $\wh L^{t}(\ol k)$ for each $t\in[T]$ and $\ol k\in[m]$. In fact, it is sufficient to lower bound the minimum observing probability of the arms in $\ol k$, that is, $\min_{j\in[n_{\ol k}]} \sum_{a\in\Nin((\ol k, j))} Z^{(t)}(a)$ (see \Cref{proof:lem realized-regretY} for detailed deduction). The case when $\ol k\in U_1$ is easier since the observing probability in the denominator can be cancelled out with some terms in the numerator (see \Cref{eqn:regretY6} and \Cref{eqn:regretY7} in \Cref{proof:lem realized-regretY}). By choosing $\gamma_{\ol k}^{(t)}(j)=\frac{x^*_{\ol k, j}}{\delta^*_{\ol k}}\alpha$ for any $t\in[T]$, $\ol k\in U_2$ and $j\in [n_{\ol k}]$, for those $\ol k\in U_2$, we can verify that
\begin{align*}
    \min_{j\in[n_{\ol k}]} \sum_{a\in\Nin((\ol k, j))} Z^{(t)}(a)&\geq \frac{1}{2}\min_{j\in[n_{\ol k}]} \sum_{(\bar k,j')\in \Nin((\bar k,j))}Y^{(t)}(\bar k)\cdot \gamma_{\bar k}(j')
    \geq \frac{Y^{(t)}(\ol k)}{2}\cdot \frac{\alpha}{\delta^*_{\ol k}}.
\end{align*} 
Then the $Y^{(t)}(\ol k)$ can be further cancelled out with $\nabla ^{-2}\Psi(Y^{(t)}(\ol k))$ in the numerator. The complete proof of this lemma is postponed in \Cref{proof:lem realized-regretY}.

\subsubsection{Regret of the Restriction Instances}\label{sec:regret-restrict}

For those restriction instances, we choose negative entropy as their potential functions.

\begin{lemma}\label{lem:realized-regretX}
	Assume $k^*\in U_2$. Let $\Phi_{\bar k}(\*x)=\sum_{j=1}^{n_{\bar{k}}} \*x(j) \log \*x(j)$. By choosing $\gamma^{(t)}(({\bar{k},j}))=\frac{x_{\bar{k},j}^*\log{n_{\bar{k}}}}{\ol \delta^*}\cdot \beta$ for every $\bar k\in U_2$ and $j\in [n_{\bar k}]$ with some  $\beta$ satisfying $1-\bar{\gamma}^{(t)}\geq \frac{1}{2}$, we have
	\[
		\frac{D_{\Phi_{\bar k^*}(\Delta_{n_{\bar k^*-1}})}}{\eta_{\bar k^*}}+\frac{\eta_{\bar k^*}}{2}\cdot\sum_{t=1}^T \E{\sup_{\*x\in [Q^{(t)}_{\bar{k}^*}, [X^{(t)}_{\bar k^*}]} \|\hat\ell^{(t)}_{\bar k^*}\|_{\grad^{-2}\Phi_{\bar k^*}(\*x)}}\le \frac{\log{n_{\bar{k}^*}}}{\eta_{\bar{k}^*}} + \frac{\eta_{\bar{k}^*}\ol\delta^*}{2\beta \log{n_{\bar{k}^*}}} T.
	\]
\end{lemma}

The main idea to prove Lemma~\ref{lem:realized-regretX} is similar to that of Lemma~\ref{lem:realized-regretY}. The proof is provided in \Cref{proof:lem realized-regretX}. 

\subsection{Adaptive Realization}\label{sec:adaptive realization}
The realization in \Cref{thm:realized-upper-bound} is based on the heuristic that the negative entropy performs well on dense restriction instances. In case the graph is ``nowhere dense'', say is of bounded in-degree, we can use Tsallis entropy as the potential function for blocks along with \emph{adaptive exploration rates} in each round to obtain \emph{optimal} regret. 

To the best of our knowledge, the idea of using adaptive exploration rate, i.e., the choice of exploration rate at each round is not uniform and depends on the distribution of the actions, is new in algorithms for bandit with graph feedback. It is also the key idea to obtain an optimal algorithm for very simple feedback graphs, e.g. directed cycles. 

The main lemma is the following one to bound the regrets contributed by restriction instances. It is instructive to compare it with Lemma~\ref{lem:realized-regretX}.

\begin{lemma}\label{lem:adapt-realized-regretX}
	Assume $k^*\in U_2$. Let $\Phi_{\bar k}(\*x)=\sum_{j=1}^{n_{\bar{k}}} -\sqrt{\*x(j)}$. By choosing $\gamma^{(t)}(({\bar{k},j}))=\frac{x_{\bar{k},j}^*}{\ol \delta^*}\cdot \beta\sum_{(\bar k,i)\in N_{\-{out}}((\bar k,j))} \sqrt{X^{(t)}_{\ol k}(i)}$ for every $\bar k\in U_2$ and $j\in [n_{\bar k}]$ with some  $\beta$ satisfying $1-\bar{\gamma}^{(t)}\geq \frac{1}{2}$, we have
	\[
		\frac{D_{\Phi_{\bar k^*}(\Delta_{n_{\bar k^*-1}})}}{\eta_{\bar k^*}}+\frac{\eta_{\bar k^*}}{2}\cdot\sum_{t=1}^T \E{\sup_{\*x\in [Q^{(t)}_{\bar{k}^*}, X^{(t)}_{\bar k^*}]} \|\hat\ell^{(t)}_{\bar k^*}\|_{\grad^{-2}\Phi_{\bar k^*}(\*x)}}\le \frac{\sqrt{n_{\bar{k}^*}}}{\eta_{\bar{k}^*}} + 2\eta_{\bar{k}^*}T\frac{\ol\delta^*}{\beta }.
	\]
\end{lemma}
\begin{proof}
	Since $X^{(t)}_{\bar k^*}$ and $Z^{(t)}$ is $\+F_{t-1}$ measurable, we have
	\begin{align*}
		\E{\sup_{\*x\in [Q^{(t)}_{\bar k^*},X^{(t)}_{\bar k^*}]} \|\hat\ell^{(t)}\|_{\grad^{-2}\Phi(\*x)}} &\leq \E{\sum_{j=1}^{n_{\bar k^*}} \frac{4X^{(t)}_{\bar k^*}(j)^{\frac{3}{2}} \*1[(\ol k^*,j)\in \Nout(A_t)]}{(\sum_{(\ol k^*,i)\in \Nin((\ol k^*,j))} Z^{(t)}((\ol k^*,i)))^2} } \\
		&= \E{\sum_{j=1}^{n_{\bar k^*}} \frac{4X^{(t)}_{\bar k^*}(j)^{\frac{3}{2}}}{(\sum_{(\ol k^*,i)\in \Nin((\ol k^*,j))} Z^{(t)}((\ol k^*,i)))^2} \E[t-1]{\*1[(\ol k^*,j)\in \Nout(A_t)]}} \\
		&= \E{\sum_{j=1}^{n_{\bar k^*}} \frac{4X^{(t)}_{\bar k^*}(j)^{\frac{3}{2}}}{\sum_{(\ol k^*,i)\in \Nin((\ol k^*,j))} Z^{(t)}((\ol k^*,i))} } \\
		&\leq \E{\sum_{j=1}^{n_{\bar k^*}}\frac{4X^{(t)}_{\bar k^*}(j)^{\frac{3}{2}}}{\sum_{(\ol k^*,i)\in \Nin((\ol k^*,j))} \beta\frac{x^*_i}{\ol\delta^*}\cdot \sqrt{X^{(t)}_{\bar k^*}(j) }} }\\
		&\leq \E{\sum_{j=1}^{n_{\bar k^*}}\frac{4X^{(t)}_{\bar k^*}(j)^{\frac{3}{2}}}{\frac{\beta}{\ol \delta^*}\sqrt{X^{(t)}_{\bar k^*}(j)}} }
		= \frac{4\ol\delta^*}{\beta }.
	\end{align*}
	By direct calculation, $\frac{D_{\Phi_{\bar k^*}(\Delta_{n_{\bar k^*-1}})}}{\eta_{\bar k^*}}\leq \frac{\sqrt{n_{\bar{k}^*}}}{\eta_{\bar{k}^*}}$.
	Thus, we have
	\[
		\frac{D_{\Phi_{\bar k^*}(\Delta_{n_{\bar k^*-1}})}}{\eta_{\bar k^*}}+\frac{\eta_{\bar k^*}}{2}\cdot\sum_{t=1}^T \E{\sup_{\*x\in [Q^{(t)}_{\bar{k}^*}, X^{(t)}_{\bar k^*}]} \|\hat\ell^{(t)}_{\bar k^*}\|_{\grad^{-2}\Phi_{\bar k^*}(\*x)}}\le \frac{\sqrt{n_{\bar{k}^*}}}{\eta_{\bar{k}^*}} + 2\eta_{\bar{k}^*}T\frac{\ol\delta^*}{\beta }.
	\]
\end{proof}

Equipped with Lemma~\ref{lem:adapt-realized-regretX} and Lemma~\ref{lem:realized-regretY}, we prove another realization of \Cref{thm:general regret}. We assume in \Cref{thm:adapt-realized-upper-bound} that the partition of the graph $G$ satisfies $U_2\neq \emptyset$. We remark that the bounds in Theorem~\ref{thm:adapt-realized-upper-bound} outperform ones in Theorem~\ref{thm:realized-upper-bound} when $G[V_{\bar k}]$ for $\bar k\in V_2$ is of bounded in-degree (and therefore they are not \emph{dense}).

\begin{theorem}\label{thm:adapt-realized-upper-bound}
	Let $G=(V,E)$ be a directed graph instance. Let $V_1,V_2,\dots,V_m$ be a legal partition of $V$ with $U_1$ and $U_2$ where $U_2\neq \emptyset$. Let $U_1^{S}\subseteq U_1$ be the indices of those singleton sets containing an arm with a self-loop and $U_1^{\ol S}=U_1\setminus U_1^S$. Then for sufficiently large $T>0$, any loss sequence $\ell^{(1)},\dots,\ell^{(T)}$ and any arm $a^*=(\bar k^*,j^*)$ in $V$, the regret of \Cref{alg:osmd-decomp} with respect to $a^*$ satisifies
	\begin{equation*}
		R_{(\bar k^*,j^*)}(T)\leq \left\{
		\begin{aligned}
			&3\cdot \tp{2\sum_{\ol k\in U_2}\sqrt{n_{\ol k}}}^{\frac{1}{3}}n_{\ol k^*}^{\frac{1}{6}}T^{\frac{2}{3}} +3\cdot 2^{\frac{2}{3}}\tp{|U_2|\sum_{\bar{k}\in U_2} (\delta^*_{\bar{k}})^2}^{\frac{1}{6}} T^{\frac{2}{3}} +4\sqrt{|U_1^S|}T^{\frac{1}{2}},& U_1^{\ol S}= \emptyset;\\
			& 6\cdot 2^{\frac{1}{3}}\tp{|U_2|\sum_{\bar{k}\in U_2} (\delta^*_{\bar{k}})^2}^{\frac{1}{6}} T^{\frac{2}{3}} + 3\cdot \tp{2\sum_{\ol k\in U_2}\sqrt{n_{\ol k}}}^{\frac{1}{3}}n_{\ol k^*}^{\frac{1}{6}}T^{\frac{2}{3}} + 4\sqrt{6|U_1^S|}T^{\frac{1}{2}} &\\ 
			&\quad\quad+2\sqrt{10\log(|U_1^{\ol S}|+1)} T^{\frac{1}{2}} + \frac{4T^{\frac{1}{3}} |U_2|^{\frac{5}{6}}}{2^{\frac{1}{3}}\tp{\sum_{\bar{k}\in U_2} (\delta^*_{\bar{k}})^2}^{\frac{1}{6}}} + \frac{\sqrt{6}}{3}T^{\frac{1}{2}},& U_1^{\ol S}\neq \emptyset.
		\end{aligned}
			\right.
	\end{equation*}
\end{theorem}

The proof of this theorem is in \Cref{sec:proof-adapt}. 


\subsection{Remark on Realization}\label{sec:remark-on-realization}

\Cref{lem:realized-regretX} and \Cref{lem:adapt-realized-regretX} correspond to two different algorithms for restriction instances and the bounds are in general not comparable. As we explained before, the parameters chosen in  \Cref{lem:realized-regretX} performs well on dense instances while those in \Cref{lem:adapt-realized-regretX} prefer sparse instances. In fact, our framework analyzed in \Cref{thm:general regret} allows each block to use their own prefered realization. Therefore, if in a given partition those weakly observable blocks are hybrid of dense ones and sparse ones, we can choose for each block either the algorithm in \Cref{lem:realized-regretX}, or the algorithm in \Cref{lem:adapt-realized-regretX}, depending on which is better.

A legal partition must be given as an input for our algorithm. A natural question is how to find a good partition beforehand. A direct solution is to regard bounds in \Cref{thm:realized-upper-bound} and \Cref{thm:adapt-realized-upper-bound} (or hybrid of them as discussed in the last paragraph) as the optimization object to find a best partition. Of course, the dependency of the regret bounds and the graph structure is complicated, and therefore the optimization problem is in general intractable. We will see in next section some natural choices of the partition already yields improved and optimal bounds. However, it is still a very interesting problem to devise an efficient way to find a good partition based on the current regret bounds in the most general setting.

\section{Applications}\label{sec:application}

We discuss applications of \Cref{thm:realized-upper-bound} and \Cref{thm:adapt-realized-upper-bound} in this section. We design optimal algorithms for $C$-corrupted strongly observable graphs (\Cref{sec:corrupted}) and graphs of bounded out-degree (\Cref{sec:directed}). We give improved algorithms when $G$ is the disjoint union of dense graphs in \Cref{sec:union}. We also formalize a conjecture regarding the lower bounds for the mini-max regret when $G$ is the disjoint union of small graphs in \Cref{sec:union}. In \Cref{sec:hypercube}, we give an improved regret bound for hypercubes by designing a non-trivial partition of the graph. 

\subsection{$C$-corrupted Strongly Observable Graphs}\label{sec:corrupted}

We say a graph is $C$-corrupted strongly observable if at most $C$ vertices in $V$ are not strongly observable. \Cref{fig:c-corrupted} illustrates a corrupted MAB and a corrupted full feedback graph.

\begin{theorem}\label{thm:c-corrupted}
If $G$ is $C$-corrupted strongly observable, then for sufficiently large $T$, any loss sequence $\ell^{(1)},\dots,\ell^{(T)}$ and any $a^*\in V$, we have
\[
    R_{a^*}(T) \leq 9\cdot \tp{4C}^{\frac{1}{3}}\cdot T^{\frac{2}{3}}.
\]
\end{theorem}
\begin{proof}
	We now define a partition of the graph and apply \Cref{thm:adapt-realized-upper-bound} to finish the proof. First let $U\subseteq V$ be the set of all the vertices that are not strongly observable. If $G[U]$ is observable, then we simply let $V\setminus U$ be the strongly observable part and let $U$ be another part. Otherwise, for every $u\in U$ that is not observable in $G[U]$, since it is weakly observable in $V$, we can pick a strongly observable vertex $v\in V\cap\Nin(u)$ and add $v$ to $U$. After this operation, $G[U]$ is weakly observable and satisfies $\abs{U}\le 2C$. Then we let $V\setminus U$ be the strongly observable part and $U$ be another part.
	
	The theorem follows from \Cref{thm:adapt-realized-upper-bound} with this partition.
\end{proof}

Note that the bound in \Cref{thm:c-corrupted} contains no $N$ factor and it is clearly optimal for constant $C$. 

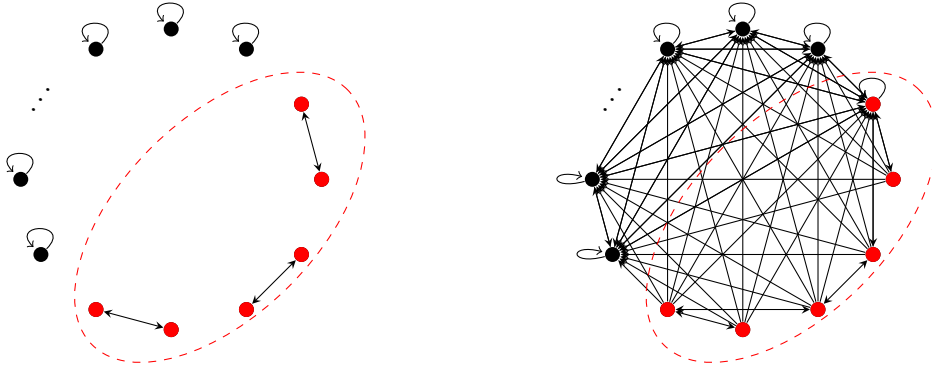
\begin{figure}[h]
	\centering
        \begin{minipage}[t]{0.45\textwidth}
            \centering
            \begin{tikzpicture}
                \tikzset{mynode/.style=fill, circle, inner sep=2pt, minimum size=3pt}
                \foreach \j in {0,1,2,...,10}{
                    \pgfmathparse{cos(30*(\j)};
                    \tikzmath{\x =\pgfmathresult;}
                    \pgfmathparse{sin(30*(\j)};
                    \tikzmath{\y =\pgfmathresult;}
                    \node[mynode] (MAB\j) at (3-2*\x,3-2*\y){};
                }
                \pgfmathparse{cos(30*11};
                \tikzmath{\x =\pgfmathresult;}
                \pgfmathparse{sin(30*11};
                \tikzmath{\y =\pgfmathresult;}
                \node[color=black] at (3-1.8*\x,3-1.8*\y){\begin{rotate}{95}$\ddots$\end{rotate}};
                \foreach \j in {0,1,8,9,10}{
                    \path[-stealth] (MAB\j) edge [loop] (MAB\j);
                }
                \foreach \j in {2,3,4,5,6,7}{
                    \pgfmathparse{cos(30*(\j)};
                    \tikzmath{\x1 =\pgfmathresult;}
                    \pgfmathparse{sin(30*(\j)};
                    \tikzmath{\y1 =\pgfmathresult;}
                    \node[mynode,red] (MAB\j) at (3-2*\x1,3-2*\y1){};
                }
                \path[stealth-stealth] (MAB2) edge (MAB3);
                \path[stealth-stealth] (MAB4) edge (MAB5);
                \path[stealth-stealth] (MAB6) edge (MAB7);
                \draw[dashed, red, rotate around={45:(4,2)}] (4.1,2.6) ellipse (2.4 and 1.3 );
            \end{tikzpicture}
        \end{minipage}
        \begin{minipage}[t]{0.45\textwidth}
            \centering
            \begin{tikzpicture}
                \tikzset{mynode/.style=fill, circle, inner sep=2pt, minimum size=3pt}
                \foreach \j in {0,1,2,...,10}{
                    \pgfmathparse{cos(30*(\j)};
                    \tikzmath{\x =\pgfmathresult;}
                    \pgfmathparse{sin(30*(\j)};
                    \tikzmath{\y =\pgfmathresult;}
                    \node[mynode] (MAB\j) at (3-2*\x,3-2*\y){};
                }
                \pgfmathparse{cos(30*11};
                \tikzmath{\x =\pgfmathresult;}
                \pgfmathparse{sin(30*11};
                \tikzmath{\y =\pgfmathresult;}
                \node[color=black] at (3-1.8*\x,3-1.8*\y){\begin{rotate}{95}$\ddots$\end{rotate}};
                \foreach \j in {7,8,9,10}{
                    \path[-stealth] (MAB\j) edge [loop] (MAB\j);
                }
                \foreach \j in {0,1}{
                    \path[-stealth] (MAB\j) edge [loop left] (MAB\j);
                }
                \foreach \i in {0,1,2,3,4,5,6,7,9,10}{
                    \path[-stealth] (MAB\i) edge (MAB8);
                }
                \foreach \i in {0,1,2,3,4,5,6,7,8,10}{
                    \path[-stealth] (MAB\i) edge (MAB9);
                }
                \foreach \i in {0,1,2,3,4,5,6,7,8,9}{
                    \path[-stealth] (MAB\i) edge (MAB10);
                }
                \foreach \i in {1,2,3,4,5,6,7,8,9,10}{
                    \path[-stealth] (MAB\i) edge (MAB0);
                }
                \foreach \i in {0,2,3,4,5,6,7,8,9,10}{
                    \path[-stealth] (MAB\i) edge (MAB1);
                }
                \foreach \i in {0,1,2,3,4,5,6,8,9,10}{
                    \path[-stealth] (MAB\i) edge (MAB7);
                }
                \foreach \j in {2,3,4,5,6,7}{
                    \pgfmathparse{cos(30*(\j)};
                    \tikzmath{\x =\pgfmathresult;}
                    \pgfmathparse{sin(30*(\j)};
                    \tikzmath{\y =\pgfmathresult;}
                    \node[mynode,red] (MAB\j) at (3-2*\x,3-2*\y){};
                }
                \path[-stealth] (MAB7) edge (MAB6);
                \path[-stealth] (MAB7) edge (MAB5);
                \path[stealth-stealth] (MAB4) edge (MAB5);
                \path[stealth-stealth] (MAB2) edge (MAB3);
                \path[stealth-stealth] (MAB2) edge (MAB4);

                \draw[dashed, red, rotate around={45:(4,2)}] (4.1,2.6) ellipse (2.4 and 1.3 );
            \end{tikzpicture}
        \end{minipage}
    \caption{Two examples of $C$-corrupted strongly observable graphs}
    \label{fig:c-corrupted}
\end{figure}

\subsection{Union of Dense Graphs}\label{sec:union}

In this section, we examine \Cref{thm:realized-upper-bound} when $G$ is the disjoint union of special graphs. We are especially interested in cases when each $G_{\bar k}$ is dense so that negative entropy is locally a good choice. These examples demonstrate that our two-stage algorithm is essential to capture the structure of these instances.

\subsubsection{Disjoint Union of Loop-less Cliques} 

Let $m\ge 2$. Assume the graph $G=(V,E)$ is the disjoint union of $G_1,\dots,G_m$ where each $G_{\bar k}=(V_{\bar k},E_{\bar k})$ is a $n_k$ loop-less clique ($E_{\bar k}=\set{(i,j)\mid i,j\in V_k, i\ne j}$). We index vertices in $V$ using $(\bar k,j)$ for $\bar k\in [m]$ and $j\in [n_{\bar k}]$ as usual. Let $N=\sum_{\bar k\in [m]} n_{\bar k}$ be the number of vertices in $G$. 
 Using the partition $V=\bigcup_{\bar k=1}^{m}V_{\bar k}$, \Cref{thm:realized-upper-bound} yields

\begin{theorem}\label{thm:union-of-cliques}
	If the weakly observable graph $G=(V,E)$ is the disjoint union of $G_1,\dots,G_m$ where each $G_{\bar k}=(V_{\bar k},E_{\bar k})$ is a $n_k$ loop-less clique. For any sufficiently large $T$, any loss vector sequence $\ell^{(1)},\dots,\ell^{(T)}$, the regret of our algorithm is 
	\[
		R(T) = O\Big(\Big(\sum_{\bar k=1}^m \log n_{\bar k}\Big)^{\frac{1}{3}}\cdot T^{\frac{2}{3}}\Big).
	\]
\end{theorem}

Note that the fractional domination number of $G$ is $2m$ and therefore previous best algorithm in~\cite{ACBDK15,CHLZ21} has regret $O\Big(\tp{m\log N}^{\frac{1}{3}}\cdot T^{\frac{2}{3}}\Big)$.

It is instructive to compare the two bounds. We can rewrite the two upper bounds respectively as
\[
O\Big(\Big(\log\prod_{\bar k=1}^m n_{\bar k}\Big)^{\frac{1}{3}}\cdot T^{\frac{2}{3}}\Big)\quad\mbox{and}\quad O\Big(\Big(\log N^m\Big)^{\frac{1}{3}}\cdot T^{\frac{2}{3}}\Big).
\]

The algorithm in~\cite{ACBDK15,CHLZ21} is simply OSMD with negative entropy and is good when the graph is dense. Therefore, when $m$ is small and each loop-less clique is of similar size (for example, when $m=2$ and $n_{1}=n_2$), their bound is close to ours . In this case, the regret contributed by restriction instances dominates, since the incidence graph $H$ is of constant size.

On the other hand, if $m$ is large, previous algorithm is much worse than ours. Suppose each $n_{\bar k}=2$, then $G$ consists of $m$ disjoint isolated edges, which is topologically close to the MAB instance\footnote{Although unlike MAB, it is weakly observable here. }. In this case, the regret of the projection instance dominates and our realization in \Cref{thm:realized-upper-bound} essentially use Tsallis entropy as the potential function, which is believed to be optimal.  For those intermediate $m$ and arbitrary value $n_{\bar k}$, our algorithm perfectly interpolates between the two extremes.

We conjecture that the bound in \Cref{thm:union-of-cliques} is optimal.

\subsubsection{Disjoint Union of Complete Bipartite Graphs}

Similarly, if $G$ is the disjoint union of $G_1,\dots, G_m$ and each $G_{\bar k}=(V_{\bar k},E_{\bar k})$ is a $\frac{n_{\bar k}}{2}+\frac{n_{\bar k}}{2}$ complete bipartite graph, then we can use the straightforward partition $V=\bigcup_{\bar k\in [m]} V_{\bar k}$ and apply \Cref{thm:realized-upper-bound} to obtain an algorithm with regret
	\[
	R(T) = O\Big(\Big(\sum_{\bar k=1}^m \log n_{\bar k}\Big)^{\frac{1}{3}}\cdot T^{\frac{2}{3}}\Big).
	\]
We know each $G_{\bar k}$ contains a $\frac{n_k}{2}$-packing independent set of size $\frac{n_k}{2}$ and therefore each $G_{\bar k}$ has regret lower bound $\Omega\tp{\tp{\log n_{\bar k}}^{\frac{1}{3}}\cdot T^{\frac{2}{3}}}$. What is the lower bound for $G$? We make the following conjecture regarding the additive property of the lower bound of this form. 

\begin{conjecture}
	If $G$ is the disjoint union of $G_1,\dots,G_{m}$ weakly observable graphs and each $G_{\bar k}$ contains an $t_{\bar k}$-packing independent set $S_{\bar k}$. Then for any algorithm, for any sufficiently large $T>0$, there exists a loss vector $\ell^{(1)},\dots,\ell^{(T)}$ yielding regret at least $\Omega\Big(\tp{\sum_{\bar k=1}^m\max\set{\log\abs{S_{\bar k}},\frac{\abs{S_{\bar k}}}{t_{\bar k}}}}^{\frac{1}{3}}\cdot T^{\frac{2}{3}}\Big)$.
\end{conjecture}

\subsection{Graphs with Bounded Degree} \label{sec:directed}

In this section, we establish the following theorem, which is \Cref{thm:out-bounded-theta} in the introduction.


\begin{theorem}\label{thm:directed-bounded-degree}
		Let $G=(V,E)$ be a weakly observable directed graph of bounded out-degree with $N$ vertices. Then for sufficiently large $T$, its mini-max regret satisfies
	\[
	R^*(T) = \Theta\tp{N^{\frac{1}{3}}\cdot T^{\frac{2}{3}}}.
	\]
\end{theorem}
\begin{proof}
	For the upper bound, we simply regard the whole graph as one block and apply \Cref{alg:osmd-decomp} with the realization in \Cref{sec:adaptive realization} on this partition. The regret of our algorithm is $O\Big(N^{\frac{1}{3}}\cdot T^{\frac{2}{3}}\Big)$ according to \Cref{thm:adapt-realized-upper-bound}.

	For the lower bound, since the out-degree of each vertex is bounded, we can find a $O(1)$-packing independent set $S$ with $\abs{S}=\Omega(N)$ in $G$ using the straightforward greedy strategy. It then follows from Proposition~\ref{prop:packing-lb} that for any algorithm, there exists some loss vectors sequence $\ell^{(1)},\dots,\ell^{(T)}$ such that the regret is $\Omega\Big(N^{\frac{1}{3}}\cdot T^{\frac{2}{3}}\Big)$. 
\end{proof}

Following the same argument above for the upper bound, we can prove \Cref{thm:bounded-degree} in the introduction. 
 In fact, the proof of this theorem implies a universal mini-max regret upper bound $O\Big(N^{\frac{1}{3}}\cdot T^{\frac{2}{3}}\Big)$ for \emph{any} weakly observable graph $G$ since one can always obtain a subgraph of $G$ with maximum in-degree $1$ by deleting edges. The operation never decrease the mini-max regret. The bound improves previous best universal upper bound $O\Big((N\log N)^{\frac{1}{3}}\cdot T^{\frac{2}{3}}\Big)$ in~\cite{ACBDK15,CHLZ21}. 

Then we have the following corollary since the in-degree and out-degree of an undirected graph are identical. This closes an open problem in~\cite{CHLZ21} where they asked for the optimal algorithm for undirected cycles.
\begin{corollary}\label{cor:undirected-bounded-degree}
		If a weakly observable graph $G=(V,E)$ with $\abs{V}=N$ is undirected and the degree of each vertex is bounded by a constant, then for sufficiently large $T$, its mini-max regret satisfies
	\[
		R^*(T) = \Theta\Big(N^{\frac{1}{3}}\cdot T^{\frac{2}{3}}\Big).
	\]
\end{corollary}

\subsection{Hypercubes}\label{sec:hypercube}

In all applications mentioned so far, the regret bounds obtained by our realizations are either provably optimal or at least we conjectured to be optimal. These algorithms are achieved by natural partition of the graph. In this section, we demonstrate that a good partition is non-trivial to find. 

A hypercube, denoted by $Q_n=(V_n,E_n)$, is an undirected graph where $V_n=\set{0,1}^n$ and two vertices are adjacent if and only if their Hamming distance is exactly $1$. We use \Cref{thm:realized-upper-bound} to prove that a hypercube $Q_n$ has regret $O\tp{\tp{\frac{N}{n}\log n}^{\frac{1}{3}}T^{\frac{2}{3}}}$ using \Cref{alg:osmd-decomp} against any $\ell^{(1)},\dots,\ell^{(T)}$ where $N=2^n$ is the number of total vertices. Note that the algorithm in~\cite{ACBDK15,CHLZ21} has regret upper bound $O\tp{N^{\frac{1}{3}}\cdot T^{\frac{2}{3}}}$ if one takes the trivial bound $\delta^*=O\tp{\frac{N}{n}}$.

\begin{theorem} \label{thm:hypercube}
	If $Q_n=(V_n,E_n)$ is a hypercube with $|V_n|=2^n=N$, then for every $T>0$, every loss vector sequence $\ell^{(1)},\dots,\ell^{(T)}$, our realization satisfies
	\[
		R(T)= O\Big(\Big(\frac{N}{n}\log n\Big)^{\frac{1}{3}}T^{\frac{2}{3}}\Big).
	\]
\end{theorem}
We use the following lemma to define a legal partition and then apply \Cref{thm:realized-upper-bound} to prove the theorem.  
	\begin{lemma}[\cite{JP90}] \label{lem: hypercube ind set}
		Let $n=2^k-1$, $k\geq 1$. Then there is a partition of $V_n$ into $n+1$ sets $S_n^{(0)}, \dots, S_n^{(n)}$ of cardinality $\frac{2^n}{n+1}$ each such that for every $0\le i \le n$, $S_n^{(i)}$ is an minimum cardinality maximal independent set (MCMIS) of $Q_n$.
	\end{lemma}
\begin{proof}
	First we construct a set $D_n\subseteq V_n$ with the following properties:
	\begin{itemize}
		\item The set $D_n$ is a dominating set of $Q_n$;
		\item The set $D_n$ can be divided into $\frac{|D_n|}{2}$ pairs of vertices where each pair of vertices are neighbors in $Q_n$ (In other words, $Q_n[D_n]$ contains a perfect matching).
	\end{itemize}
	If $n=2^k-1$ for a positive integer $k$, let $D_n =S_n^{(0)}\cup S_n^{(1)}$ where $\set{S_n^{(0)}, \dots, S_n^{(n)}}$ be the partition in Lemma~\ref{lem: hypercube ind set}. We now prove that such $D_n$ satisfies above properties. For both $S_n^{(0)}$ and $S_n^{(1)}$ are maximal independent sets, every vertex in $V_n$ is connected to some vertices in $D_n$. Thus $D_n$ satisfies the first dominating property. Obviously, $|D_n|=\frac{2^{n+1}}{2^k}$ is even. For every vertex in $S_n^{(i)}$, $i\in \set{0}\cup [n]$, it has at least one neighbor in every other blocks. Note that every vertex in $Q_n$ has $n$ neighbors. Thus, each vertex in $S_n^{(0)}$ is connected with exactly one vertex in $S_n^{(1)}$ and vice versa. So $D_n$ satisfies the second pairing property.

	Then we construct such $D_n$ for general $n\geq 1$ by induction. Assume that we have such a $D_n$ for $Q_n$ where $2^k-1\leq n< 2^{k+1}-2$ and $k$ is a positive integer. We denote a binary string ending with $1$ in $Q_n$ by $\sim 1$ and similarly define $\sim 0$. We extend $\sim 1$ to $\sim 01$ and $\sim 10$,  $\sim 0$ to $\sim 00$ and $\sim 11$ to get $Q_{n+1}$. We form $D_{n+1}$ by extending $D_n$ in this way. Note that the two extensions of each string in $V_n\setminus D_n$ can be dominated by some vertices in $D_{n+1}$. For each pair in $D_n$, the four extended strings can form two pairs. Thus $D_{n+1}$ satisfies the two properties as well.
	
	It follows from above analysis that each $D_n$ has $\frac{2^n}{2^k}$ pairs of vertices. Then we construct a partition of $Q_n$ to feed \Cref{thm:realized-upper-bound}: We prepare $\frac{2^n}{2^k}$ empty blocks and put each pair of vertices in $D_n$ into each block without repetition. For each vertex $v\in V_n\setminus D_n$, there must be one vertex $u\in D_n$ which is adjacent to $v$ (if there exists more than one such vertex, choose any one of them). Then we put $v$ into the block containing $u$. We know that every vertex can be put in one block, and each block contains at most $2n$ vertices. This yields that there are at least $\frac{2^n}{2^{k+1}}$ blocks with not less than $n$ vertices since otherwise the total vertex number would be less that $2^{n}$. The fractional domination number of each block is at most $2$ for a partition constructed in the above way. We can then apply \Cref{alg:osmd-decomp} on $Q_n$ with this partition. By \Cref{thm:realized-upper-bound}, we have that $R(T)=O\Big(\tp{\frac{N}{n}\log n}^{\frac{1}{3}}T^{\frac{2}{3}}\Big)$ where $N=|V_n|=2^n$.
\end{proof}

\section{Conclusions and Future Work}

In this article, we introduced a new two-level algorithmic framework for solving bandit with graph feedback. Conceptually, we demonstrated that the hierarchical view of the graph structure is essential towards an optimal algorithm. Technically, we proved a regret decomposition theorem characterizing the interplay between the parts of the graph in terms of their contributed regrets.  Moreover, we further introduced sophisticated realizations of the framework which yields improved and optimal regret in many cases. The technique developed in these realizations might find applications in other problems. 

A few interesting problems regarding the performance of the framework remain. Our algorithm relies on a partition of the graph and it is quite challenging to determine the best partition for a given graph. As discussed in \Cref{sec:remark-on-realization}, finding the best partition achieving minimum regret in \Cref{thm:realized-upper-bound} and \Cref{thm:adapt-realized-upper-bound} in general is already a computational heavy task. It is still possible that an efficient \emph{approximation algorithm} for a certain relaxation of the optimization problem exists. 

Another interesting problem is to confirm the optimality of some regret bounds achieved in the article, especially those discussed in \Cref{sec:union}.

\bibliography{ref_graph_decomp}
\bibliographystyle{alpha}


\appendix
\section{Proof of Lemma~\ref{lem:regret}, Lemma~\ref{lem:regretY} and Lemma~\ref{lem:regretX}}\label{sec:proof}

\subsection{Proof of Lemma~\ref{lem:regret}}\label{proof:lem regret}
\begin{proof}
	It is routine to have
	\begin{align*}
		R_{(\bar k^*,j^*)}(T)& = \E{\sum_{t=1}^T \tp{\ell^{(t)}(A_t) - \ell^{(t)}(a^*)}}= \sum_{t=1}^T\E{\E[t-1]{\ell^{(t)}(A_t) - \ell^{(t)}(a^*)}}\\
		&= \sum_{t=1}^T\E{\inner{\ell^{(t)}}{Z^{(t)}-\*e^{[N]}_{a^*}}}=\sum_{t=1}^T\E{\E[t-1]{\inner{\ell^{(t)}}{Z^{(t)}-\*e^{[N]}_{a^*}}}}\\
		&=\sum_{t=1}^T\E{\E[t-1]{\inner{\hat\ell^{(t)}}{Z^{(t)}-\*e^{[N]}_{a^*}}}}=\sum_{t=1}^T\E{\inner{\hat\ell^{(t)}}{Z^{(t)}-\*e^{[N]}_{a^*}}}.
	\end{align*}
	So it suffices to bound $\E{\inner{\hat\ell^{(t)}}{Z^{(t)}-\*e^{[N]}_{a^*}}}$. We now show that it can be decomposed into four parts. We have for every $t\in [T]$,
	\begin{align*}
		&\phantom{{}={}}\E{\inner{\hat\ell^{(t)}}{Z^{(t)}-\*e^{[N]}_{a^*}}}\\
		&=\E{\sum_{\bar k\in[m]}\sum_{j\in[n_{\bar k}]} \hat\ell^{(t)}((\bar k,j))\cdot Z^{(t)}((\bar k,j)) - \hat\ell^{(t)}(a^*)}\\
		&=\E{\sum_{\bar k\in U_2}\sum_{j\in[n_{\bar k}]} \hat\ell^{(t)}((\bar k,j))\cdot \tp{(1-\bar {\gamma}^{(t)})\cdot Y^{(t)}(\bar k)\cdot \tilde X^{(t)}_{\bar k}(j) + \gamma^{(t)}((\bar k,j))}}\\ 
		&\quad\quad+ \E{\sum_{\bar k\in U_1}\hat\ell^{(t)}((\bar k,1))\cdot \tp{(1-\bar{\gamma}^{(t)})\cdot Y^{(t)}(\bar k) + \gamma^{(t)}((\bar k,1))} - \hat\ell^{(t)}(a^*)}\\
		&\le\E{\sum_{\bar k\in U_2}Y^{(t)}(\bar k)\sum_{j\in[n_{\bar k}]}\tilde X^{(t)}_{\bar k}(j)\cdot \hat\ell^{(t)}((\bar k,j)) + \sum_{\bar k\in U_1}Y^{(t)}(\bar k) \cdot \hat\ell^{(t)}((\bar k,1))} \\ 
		&\quad\quad +\sum_{\bar k\in [m]}\sum_{j\in[n_{\bar k}]}\gamma^{(t)}((\bar k,j))-\E{\hat\ell^{(t)}(a^*)}\\
		&= \E{\sum_{\bar k\in[m]}Y^{(t)}(\bar k)\cdot \wh L^{(t)}(\bar k)}+\sum_{\bar k\in [m]}\sum_{j\in[n_{\bar k}]}\gamma^{(t)}((\bar k,j))-\E{\hat\ell^{(t)}_{\bar k^*}(j^*)}\\
		&= \E{\inner{\wh L^{(t)}}{Y^{(t)}-\*e^{[m]}_{\bar k^*}}+\inner{\wh L^{(t)}}{\*e^{[m]}_{\bar k^*}}} + \sum_{\bar k\in [m]}\sum_{j\in[n_{\bar k}]}\gamma^{(t)}((\bar k,j))-\E{\inner{\hat\ell^{(t)}_{\bar k^*}}{\*e^{[n_{\bar k^*}]}_{j^*}}}.\\
	\end{align*}
	The lemma follows by observing that
	\begin{itemize}
	\item If $\bar k\in U_2$, then
	\begin{align*}
		\E{\inner{\wh L^{(t)}}{\*e^{[m]}_{\bar k^*}}} &= \E{\wh L^{(t)}(\bar k^*)} = \E{\sum_{j\in [n_{\bar k^*}]} \tilde X^{(t)}_{\bar k^*}(j)\cdot \hat\ell^{(t)}_{\bar k}(j)} = \E{\inner{\hat\ell^{(t)}_{\bar k^*}}{\tilde X^{(t)}_{\bar k^*}}}
		 \le \E{\inner{\hat\ell^{(t)}_{\bar k^*}}{X^{(t)}_{\bar k^*}}} + \sum_{j\in [n_{\bar k^*}]}\gamma^{(t)}_{\bar k^*}(j).
	\end{align*}
	\item If $\bar k\in U_1$, then
	\[
	\E{\inner{\wh L^{(t)}}{\*e^{[m]}_{\bar k^*}}} = \E{\wh L^{(t)}(\bar k^*)} = \E{\hat\ell^{(t)}_{\bar k^*}(1)}=\E{\inner{\hat\ell^{(t)}_{\bar k^*}}{\*e^{[n_{\bar k^*}]}_{\bar j^*}}}.
	\]
	\end{itemize}
\end{proof}
\subsection{Proof of Lemma~\ref{lem:regretY}} \label{proof:lem regretY}
\begin{proof}
	First note that $\sum_{t\in [T]} \E{\inner{\wh L^{(t)}}{Y^{(t)}-\*e^{[m]}_{\bar k^*}}} 
	 =\sum_{t\in {T}} \E{{\inner{\wh L^{(t)}-c^{(t)}\cdot \*1^{[m]}}{Y^{(t)}-\*e^{[m]}_{\bar k^*}}}}$ since $c^{(t)}\cdot \*1^{[m]}$ is constant vector. Therefore, our updates on $Y^{(t)}$ in \Cref{alg:osmd-decomp} are equivalent to applying OSMD with loss vector $(\wh L^{(t)})' = \wh L^{(t)}-c^{(t)}\cdot \*1^{[m]}$ and potential function $\Psi$. Therefore, it follows from Proposition~\ref{prop:osmd} (by taking $\eta=1$) that
	 \[
	 \sum_{t\in [T]} \E{\inner{\wh L^{(t)}}{Y^{(t)}-\*e^{[m]}_{\bar k^*}}}\le D_{\Psi}(\Delta_{m-1})+ \frac{1}{2}\sum_{t=1}^T\E{\sup_{\*y\in [W^{(t)}, Y^{(t)}]}\|(\wh L^{(t)})'\|_{\grad^{-2}\Psi(\*y)}}.
	 \]
\end{proof}
We remark that in the proof above the choice $\eta=1$ is without loss of optimality since we essentially hide the choice of ``learning rate'' in the potential function $\Psi$.
\subsection{Proof of Lemma~\ref{lem:regretX}} \label{proof:lem regretX}
\begin{proof}
	Similarly our updating of $X^{(t)}_{\bar k^*}$  in \Cref{alg:osmd-decomp} is equivalent to applying OSMD on the restricted instance $G[V_{\bar k^*}]$ with loss vectors $\ell^{(1)}_{\bar k^*}, \ell^{(2)}_{\bar k^*},\dots,\ell^{(T)}_{\bar k^*}$. For every $t\in [T]$, the vector $\hat\ell^{(t)}_{\bar k^*}$ is an unbiased estimator of $\ell^{(t)}_{\bar k^*}$. With this observation, the lemma directly follows from Proposition~\ref{prop:osmd}.
\end{proof}

\section{Proof of Lemma~\ref{lem:realized-regretY} and Lemma~\ref{lem:realized-regretX}}
\begin{lemma}\label{lem: monotone}
	Let $\Psi(\*y) = \sum_{\bar k\in U_2} \frac{-\sqrt{\*y(\bar k)}}{\eta} + \sum_{\bar k\in U_1^S} \frac{-\sqrt{\*y(\bar k)}}{\eta_S} + \sum_{\bar k\in U_1^{\ol S}} \frac{\*y(\bar k)\log(\*y(\bar k))}{\eta_{\ol S}}$ and $W=\arg\min_{\*a\in \^R^m} \inner{\*a}{L'}+B_{\Psi}(\*a,Y)$. If $L'(i)\cdot\max\set{\eta,\eta_S,\eta_{\bar S}}\geq -\frac{1}{4}$, then $W(i)\leq 4Y(i)$ for each $i\in[m]$. 
\end{lemma}
\begin{proof}
	Since $\Psi(\*y)$ is coordinate-wise separable, we can consider each coordinate independently, that is, 
	\begin{equation}
		W(i)=\arg\min_{x\in \^R} L'(i)\cdot x+B_{\Psi_i}(x,Y(i)).\label{eq:monotone}
	\end{equation} Here $\Psi_i$ can be negative entropy or Tsallis depending on the type of vertex $i$. Compute the derivation of the RHS of \Cref{eq:monotone}, we have 
	\begin{equation}
		L'(i) + \grad \Psi_i(W(i)) -\grad \Psi_i(Y(i))=0. \label{eq:monotone2}
	\end{equation}
	Let $\eta_0=\eta$ if $i\in U_2$, $\eta_0=\eta_S$ if $i\in U_1^S$ and $\eta_0=\eta_{\bar S}$ if $i\in U_1^{\bar S}$. When $Y(i)=0$, we can verify that $W(i)=0$. Since $Y(i)=0$, $\grad \Psi_i(Y(i))=-\infty$. If $W(i)\neq 0$, then the left hand side of \Cref{eq:monotone2} is $-\infty$. This is in contradiction with the fact that the left hand side of \Cref{eq:monotone2} equals to $0$. In this case, it is trivial to have $W(i)\leq 4Y(i)$. Then we consider the situation that $Y(i)\neq 0$.
	\begin{itemize}
		\item If $\Psi_i(x)=\frac{-\sqrt{x}}{\eta_0}$, \Cref{eq:monotone2} is equivalent to $2\eta_0L'(i) -\frac{1}{\sqrt{W_i}}+ \frac{1}{\sqrt{Y_i}}=0$. That is 
		$$W(i)=\frac{Y(i)}{\tp{2\eta_0L'(i)Y(i) +1 }^2}. $$
		\item If $\Psi_i(x)=\frac{x\log(x)}{\eta_0}$, \Cref{eq:monotone2} is equivalent to $2\eta_0L'(i)+\log\frac{W(i)}{Y(i)}=0$. That is, $$W(i)=Y(i)\exp{-2\eta_0L'(i)}.$$
	\end{itemize}
	Since $\eta_0L'(i)\geq -\frac{1}{4}$, we have $W(i)\leq 4Y(i)$ for $\Psi_i(x)=\frac{-\sqrt{x}}{\eta_0}$ and $W(i)\leq \sqrt{e}Y(i)\leq 4Y(i)$ for $\Psi_i(x)=\frac{x\log(x))}{\eta_0}$. 
\end{proof}
\subsection{Proof of Lemma~\ref{lem:realized-regretY}} \label{proof:lem realized-regretY}
\begin{proof}
	First we prove the lemma when $U_1^{\ol S}\neq \emptyset$. 
	Note that $Z^{(t)}$ and $Y^{(t)}$ are $\+F_{t-1}$-measurable. Lemma~\ref{lem: monotone} shows $W^{(t)}(i)\le 4Y^{(t)}(i)$ for every $i\in [m]$. Therefore, we have for every $t\in [T]$, 
	\begin{align}
		&\phantom{{}={}}\E{\sup_{\*y\in [W^{(t)},Y^{(t)}]} \| (\wh L^{(t)})' \|^2_{\grad^{-2}\Psi(\*y)}}  \notag\\
		&=\E{\E[t-1]{\sup_{\*y\in[W^{(t)}, Y^{(t)}]}\tp{\sum_{\bar k\in U_2}(\wh L^{(t)})'(\bar k)^2\cdot 4 \eta \*y(\bar k)^{\frac{3}{2}} + \sum_{\bar k\in U_1^S}(\wh L^{(t)})'(\bar k)^2\cdot 4\eta_{S} \*y(\bar k)^{\frac{3}{2}} 		
		+ \sum_{\bar k\in U_1^{\ol S}}(\wh L^{(t)})'(\bar k)^2\cdot \eta_{\ol S}\*y(\bar k) }}}\notag \\
		& \leq 4\E{\sum_{\bar k\in U_2}\E[t-1]{(\wh L^{(t)})'(\bar k)^2}\cdot 4 \eta Y^{(t)}(\bar k)^{\frac{3}{2}} + \sum_{\bar k\in U_1^S}\E[t-1]{(\wh L^{(t)})'(\bar k)^2}\cdot 4  \eta_S Y^{(t)}(\bar k)^{\frac{3}{2}} + \sum_{\bar k\in U_1^{\ol S}}\E[t-1]{(\wh L^{(t)})'(\bar k)^2}\cdot \eta_{\ol S}Y^{(t)}(\bar k) }. \label{eqn:regretY1}
	\end{align}
	By direct calculation we have
	\begin{align}
		\sum_{\bar k\in U_2}\E[t-1]{(\wh L^{(t)})'(\bar k)^2}\cdot Y^{(t)}(\bar k)^{\frac{3}{2}} &=\sum_{\bar k\in U_2}\E[t-1]{\tp{\wh L^{(t)}(\bar k)-c^{(t)}}^2}\cdot  Y^{(t)}(\bar k)^{\frac{3}{2}}\leq \sum_{\bar k\in U_2 }\E[t-1]{\wh L^{(t)}(\bar k)^2+\tp{c^{(t)}}^2}\cdot Y^{(t)}(\bar k)^{\frac{3}{2}}. \label{eqn:regretY2}
	\end{align}
	By the definition of $\wh L^{(t)}$ and $c^{(t)}$, we have
	\begin{align}
		&\phantom{{}={}}\sum_{\bar k\in U_2} \E[t-1]{\tp{c^{(t)}}^2}\cdot Y^{(t)}(\bar k)^{\frac{3}{2}}\notag\\
		& =  \sum_{\bar k\in U_2 } \E[t-1]{\tp{\sum_{\bar i\in U_1^{\ol S}} \wh L^{(t)}(\bar i) \cdot Y^{(t)}(\bar i)}^2}\cdot Y^{(t)}(\bar k)^{\frac{3}{2}} \notag\\
		&\leq  \sum_{\bar k\in U_2} \E[t-1]{\sum_{\bar i\in U_1^{\ol S}} (\wh L^{(t)}(\bar i))^2 \cdot Y^{(t)}(\bar i)}\cdot Y^{(t)}(\bar k)^{\frac{3}{2}}\notag\\
		&= \sum_{\bar k\in U_2} \E[t-1]{\sum_{\bar i\in U_1^{\ol S}} \frac{\*1[(\bar i,1)\in \Nout(A_t)]}{\tp{\sum_{a\in \Nin((\bar i,1))}Z^{(t)}(a)}^2} \cdot Y^{(t)}(\bar i)}\cdot Y^{(t)}(\bar k)^{\frac{3}{2}} \notag\\
		&=  \sum_{\bar k\in U_2} \sum_{\bar i\in U_1^{\ol S}} \frac{1}{1-(1-\bar{\gamma})Y^{(t)}(\bar i)} \cdot Y^{(t)}(\bar i) \cdot Y^{(t)}(\bar k)^{\frac{3}{2}} \notag \\
		&\leq  \sum_{\bar k\in U_2}Y^{(t)}(\bar k)^{\frac{1}{2}} \sum_{\bar i\in U_1^{\ol S}} \frac{1}{1-Y^{(t)}(\bar i)} \cdot Y^{(t)}(\bar i) \cdot Y^{(t)}(\bar k) \notag \\
		&\leq  \sum_{\bar k\in U_2}Y^{(t)}(\bar k)^{\frac{1}{2}} \sum_{\bar i\in U_1^{\ol S}}   Y^{(t)}(\bar i) 	\leq  \sqrt{|U_2 |}. \label{eqn:regretY3}
	\end{align}
	Similarly we have
	\begin{equation}
		\sum_{\bar k\in U_1^S}\E[t-1]{(\wh L^{(t)})'(\bar k)^2}\cdot Y^{(t)}(\bar k)^{\frac{3}{2}}\leq \sqrt{|U_1^S |} + \sum_{\bar k\in U_1^S }\E[t-1]{\wh L^{(t)}(\bar k)^2}\cdot Y^{(t)}(\bar k)^{\frac{3}{2}}. \label{eqn:regretY4}
	\end{equation}
	Note that for every $\bar k\in U_2$, $X^{(t)}_{\bar k}$ is $\+F_{t-1}$-measurable, we have
	\begin{align}
	&\phantom{{}={}}\sum_{\bar k\in U_2}\E[t-1]{\wh L^{(t)}(\bar k)^2}\cdot Y^{(t)}(\bar k)^{\frac{3}{2}} \notag\\
	&=\E[t-1]{\sum_{\bar k\in U_2}\tp{\sum_{j\in[n_{\bar k}]}\tilde{X}^{(t)}_{\bar k}(j)\cdot\hat\ell^{(t)}_{\bar k}(j)}^2\cdot Y^{(t)}(\bar k)^{\frac{3}{2}}}\notag \\
	&= \sum_{\bar k\in U_2} Y^{(t)}(\bar k)^{\frac{3}{2}}\cdot \E[t-1]{\tp{\sum_{j\in[n_{\bar k}]}\tilde{X}^{(t)}_{\bar k}(j)\cdot\hat\ell^{(t)}_{\bar k}(j)}^2}\notag \\
	&\le \sum_{\bar k\in U_2} Y^{(t)}(\bar k)^{\frac{3}{2}}\cdot\E[t-1]{\sum_{j\in[n_{\bar k}]}\tilde{X}^{(t)}_{\bar k}(j)\cdot \hat\ell^{(t)}_{\bar k}(j)^2}\notag\\
	&= \sum_{\bar k\in U_2}Y^{(t)}(\bar k)^{\frac{3}{2}}\sum_{j\in [n_{\bar k}]}\tilde{X}^{(t)}_{\bar k}(j)\cdot \E[t-1]{\frac{\*1[(\bar k,j)\in \Nout(A_t)]}{\tp{\sum_{a\in \Nin((\bar k,j))}Z^{(t)}(a)}^2}}\notag \\
	&= \sum_{\bar k\in U_2}Y^{(t)}(\bar k)^{\frac{3}{2}}\sum_{j\in [n_{\bar k}]}\frac{\tilde{X}^{(t)}_{\bar k}(j)}{\sum_{a\in \Nin((\bar k,j))}Z^{(t)}(a)}\notag \\
	&\le \sum_{\bar k\in U_2}Y^{(t)}(\bar k)^{\frac{3}{2}}\sum_{j\in [n_{\bar k}]}\frac{2\tilde{X}^{(t)}_{\bar k}(j)}{\sum_{(\bar k,j')\in \Nin((\bar k,j))}Y^{(t)}(\bar k)\cdot \gamma_{\bar k}(j')}\notag \\	
	&\le 2\sum_{\bar k\in U_2}Y^{(t)}(\bar k)^{\frac{1}{2}}\frac{\delta^*_{\bar k}}{\alpha}\notag \\	
	&\le \frac{2\sqrt{\sum_{\bar k\in U_2}(\delta^*_{\bar k})^2}}{\alpha}. \label{eqn:regretY5}
	\end{align}
	For vertices in $U_1^S $, we have
	\begin{align}
		&\phantom{{}={}}\sum_{\bar k\in U_1^S}\E[t-1]{\wh L^{(t)}(\bar k)^2}\cdot Y^{(t)}(\bar k)^{\frac{3}{2}}
		=\sum_{\bar k\in U_1^S}\E[t-1]{\hat\ell^{(t)}_{\bar k}(1)^2}\cdot Y^{(t)}(\bar k)^{\frac{3}{2}}\notag \\
		&= \sum_{\bar k\in U_1^S}\E[t-1]{\frac{\*1[(\bar k,1)\in \Nout(A_t)]}{\tp{\sum_{a\in \Nin((\bar k,1))}Z^{(t)}(a)}^2}}\cdot Y^{(t)}(\bar k)^{\frac{3}{2}} 
		= \sum_{\bar k\in U_1^S} Y^{(t)}(\bar k)^{\frac{3}{2}} \cdot \frac{1}{\sum_{a\in \Nin((\bar k,1))}Z^{(t)}(a)} \notag\\
		&\leq \sum_{\bar k\in U_1^S} Y^{(t)}(\bar k)^{\frac{3}{2}} \cdot \frac{1}{(1-\bar{\gamma})Y^{(t)}(\bar k)}
		=  \frac{1}{1-\bar {\gamma}}\sum_{\bar k\in U_1^S} Y^{(t)}(\bar k)^{\frac{1}{2}} 
		\leq  \frac{1}{1-\bar {\gamma}}\sqrt{|U_1^S|}.\label{eqn:regretY6}
	\end{align}
	For vertices in $U_1^{\ol S}$, we have
	\begin{align}
	    &\phantom{{}={}}\sum_{\bar k\in U_1^{\ol S}}\E[t-1]{(\wh L^{(t)})'(\bar k)^2}\cdot Y^{(t)}(\bar k) \notag\\
	    &=  \sum_{\bar k\in U_1^{\ol S}}\E[t-1]{\tp{\wh L^{(t)}(\bar k) - c^{(t)}}^2}\cdot Y^{(t)}(\bar k)\notag \\
		&=\E[t-1]{\sum_{\bar k\in U_1^{\ol S}} \tp{Y^{(t)}(\bar k)\cdot \wh L^{(t)}(\bar k)^2 +  Y^{(t)}(\bar k)\cdot \tp{c^{(t)}}^2 -2Y^{(t)}(\bar k)\cdot \wh L^{(t)}(\bar k)\cdot c^{(t)}}} \notag \\
		&\leq \E[t-1]{\sum_{\bar k\in U_1^{\ol S}} Y^{(t)}(\bar k)\cdot \wh L^{(t)}(\bar k)^2} + \E[t-1]{\tp{c^{(t)}}^2} -2\E[t-1]{\tp{c^{(t)}}^2} \notag \\
		&=  \E[t-1]{\sum_{\bar k\in U_1^{\ol S}}Y^{(t)}(\bar k)\wh L^{(t)}(\bar k)^2 - \tp{c^{(t)}}^2} \notag\\
		&\leq  \E[t-1]{\sum_{\bar k\in U_1^{\ol S}}Y^{(t)}(\bar k)\wh L^{(t)}(\bar k)^2 - \sum_{\bar k\in U_1^{\ol S}}Y^{(t)}(\bar k)^2\wh L^{(t)}(\bar k)^2  } \notag \\
		&=  \E[t-1]{\sum_{\bar k\in U_1^{\ol S}}Y^{(t)}(\bar k)\tp{1-Y^{(t)}(\bar k)}\wh L^{(t)}(\bar k)^2 }\notag \\
		&= \sum_{\bar k\in U_1^{\ol S}}\E[t-1]{\hat\ell^{(t)}_{\bar k}(1)^2}\cdot Y^{(t)}(\bar k)\tp{1-Y^{(t)}(\bar k)} \notag \\
		&= \sum_{\bar k\in U_1^{\ol S}}  Y^{(t)}(\bar k)\tp{1-Y^{(t)}(\bar k)} \cdot \E[t-1]{\frac{\*1[(\bar k,1)\in \Nout(A_t)]}{\tp{\sum_{a\in \Nin((\bar k,1))}Z^{(t)}(a)}^2}} \notag \\
		&= \sum_{\bar k\in U_1^{\ol S}}  Y^{(t)}(\bar k)\tp{1-Y^{(t)}(\bar k)} \cdot \frac{1}{\sum_{a\in \Nin((\bar k,1))}Z^{(t)}(a)} \notag\\
		&\leq  \sum_{\bar k\in U_1^{\ol S}}  Y^{(t)}(\bar k)\tp{1-Y^{(t)}(\bar k)} \cdot \frac{1}{1-Y^{(t)}(\bar k)}
		\leq  1.\label{eqn:regretY7}
	\end{align}
	Combining \Cref{eqn:regretY1}, \Cref{eqn:regretY2}, \Cref{eqn:regretY3}, \Cref{eqn:regretY4},\Cref{eqn:regretY5}, \Cref{eqn:regretY6}, \Cref{eqn:regretY7}, we have
	\begin{align*}
		 &\phantom{{}={}}\frac{1}{2}\E{ \sup_{\*y\in [W^{(t)},Y^{(t)}]} \|(\wh L^{(t)})' \|^2_{\grad^{-2}\Psi(\*y)}} \\
		 &\leq {2}\E{\sum_{\bar k\in U_2}\E[t-1]{\wh L^{(t)}(\bar k)^2+\tp{c^{(t)}}^2}\cdot 4\eta Y^{(t)}(\bar k)^{\frac{3}{2}}} \\
		 &\quad\quad  + {2}\E{\sum_{\bar k\in U_1^S}\E[t-1]{\wh L^{(t)}(\bar k)^2+\tp{c^{(t)}}^2}\cdot 4\eta_S Y^{(t)}(\bar k)^{\frac{3}{2}}} \\
		 &\quad\quad +{2}\E{\sum_{\bar k\in U_1^{\ol S}}\E[t-1]{(\wh L^{(t)})'(\bar k)^2}\cdot \eta_{\ol S} Y^{(t)}(\bar k)} \\
		 &\leq {2} \E{\sum_{\bar k\in U_2}\E[t-1]{(\wh L^{(t)}(\bar k))^2}\cdot 4\eta  Y^{(t)}(\bar k)^{\frac{3}{2}}}
		  + {2}\E{\sum_{\bar k\in U_1^S}\E[t-1]{(\wh L^{(t)}(\bar k))^2}\cdot 4\eta_S Y^{(t)}(\bar k)^{\frac{3}{2}}} \\
		 &\quad\quad + {2}\E{\sum_{\bar k\in U_1^{\ol S}}\E[t-1]{(\wh L^{(t)})'(\bar k)^2}\cdot \eta_{\ol S}Y^{(t)}(\bar k)} + 8\eta\sqrt{|U_2 |} + 8\eta_S \sqrt{|U_1^S |}\\
		 &\leq \eta \frac{16\sqrt{\sum_{\bar k\in U_2}(\delta^*_{\bar k})^2}}{\alpha} + 8\tp{1+ \frac{1}{1-\bar {\gamma}}}\eta_S \sqrt{|U_1^S|}+ {2}\eta_{\ol S} + 8\eta\sqrt{|U_2 |}.
	\end{align*}

	On the other hand, we have that for any $\*y\in \Delta_{m-1}$, $\Psi(\*y)\leq 0$. Thus, $D_\Psi(\Delta_{m-1}) \le \max_{\*y\in \Delta_{m-1}} \abs{\Psi(\*y)} \leq \frac{\sqrt{|U_2|}}{\eta} + \frac{\sqrt{|U_1^S|}}{\eta_S} + \frac{\log(|U_1^{\ol S}|+1)}{\eta_{\ol S}}$. Then we obtain
	\begin{align*}
		&\phantom{{}={}} D_{\Psi}(\Delta_{m-1}) + \frac{1}{2}\sum_{t=1}^T\E{\sup_{\*y\in [W^{(t)}, Y^{(t)}]}\tp{\|(\wh L^{(t)})'\|_{\grad^{-2}\Psi(\*y)} }} \\ 
		& \le \frac{\sqrt{|U_2|}}{\eta} + \frac{\sqrt{|U_1^S|}}{\eta_S} + \frac{\log(|U_1^{\ol S}|+1)}{\eta_{\ol S}}+  16\eta T \frac{\sqrt{\sum_{\bar k\in U_2}(\delta^*_{\bar k})^2}}{\alpha} \\
		&\quad\quad + 8\tp{1+\frac{1}{1-\bar {\gamma}}}\eta_S T \sqrt{|U_1^S|}+ {2}\eta_{\ol S}T  + 8\eta T \sqrt{|U_2 |}.
	\end{align*}
	The lemma for $U_1^{\ol S}= \emptyset$ is proved by similar analysis except that $c^{(t)}=0$ which yields $\sum_{\bar k\in U_1^S }$ $\E[t-1]{(\wh L^{(t)})'(\bar k)^2}$ $\cdot Y^{(t)}(\bar k)^{\frac{3}{2}} = \sum_{\bar k\in U_1^S}\E[t-1]{\wh L^{(t)}(\bar k)^2} \cdot Y^{(t)}(\bar k)^{\frac{3}{2}}$ and $\sum_{\bar k\in U_2}$ $\E[t-1]{(\wh L^{(t)})'(\bar k)^2}$ $\cdot Y^{(t)}(\bar k)^{\frac{3}{2}} = \sum_{\bar k\in U_2}\E[t-1]{\wh L^{(t)}(\bar k)^2}\cdot Y^{(t)}(\bar k)^{\frac{3}{2}}$ in this situation.
\end{proof}

\subsection{Proof of Lemma~\ref{lem:realized-regretX}} \label{proof:lem realized-regretX}
\begin{proof}
We write $\gamma^{(t)}(({\bar{k},j}))$ as $\gamma(({\bar{k},j}))$ and write $\ol \gamma^{(t)}$ as $\ol \gamma$ in the proof as they are invariant over time. With similar analysis in Lemma~\ref{lem:realized-regretY}, we have
	\begin{align}
		\E{\sup_{\*z\in [Q^{(t)}_{\bar{k}^*}, X^{(t)}_{\bar{k}^*}]} \|\hat{\ell}^{(t)}_{\bar{k}^*} \|_{\nabla^{-2}\Phi_{\bar{k}^*}(\*z)}} &\leq  \E{\sum_{j=1}^{n_{\bar{k}^*}} \frac{X^{(t)}_{\bar{k}^*}(j) \*1[(\bar{k}^*,j)\in N_{out}(A_t)]}{\tp{\sum_{a\in \Nin((\bar{k}^*,j))}Z^{(t)}(a)}^2}} \notag \\
        &= \E{\sum_{j=1}^{n_{\bar{k}^*}} \frac{X^{(t)}_{\bar{k}^*}(j) \E[t-1]{1[(\bar{k}^*,j)\in N_{out}(A_t)]}}{\tp{\sum_{a\in \Nin((\bar{k}^*,j))}Z^{(t)}(a)}^2}} \notag \\
        &= \E{\sum_{j=1}^{n_{\bar{k}^*}} \frac{X^{(t)}_{\bar{k}^*}(j)}{\sum_{a\in \Nin((\bar{k}^*,j))}Z^{(t)}(a)}}. \label{equation: lemma inside block 2}
	\end{align}

	It remains to give a lower bound to the denominator $\sum_{a\in \Nin((\bar{k}^*,j))}Z^{(t)}(a)$ which is the probability that $(\bar{k}^*,j)$ is observed in round $t$:
    \begin{align}
        \sum_{j=1}^{n_{\bar{k}^*}} \frac{X^{(t)}_{\bar{k}^*}(j)}{\sum_{a\in \Nin((\bar{k}^*,j))}Z^{(t)}(a)} &\leq
        \sum_{j=1}^{n_{\bar{k}^*}} \frac{X^{(t)}_{\bar{k}^*}(j)}{\sum_{a\in \Nin((\bar{k}^*,j))\cap V_{\bar{k}^*}}Z^{(t)}(a)} \notag \\
         & = \sum_{j=1}^{n_{\bar{k}^*}} \frac{X^{(t)}_{\bar{k}^*}(j)}{\sum_{(\bar{k}^*, s)\in \Nin((\bar{k}^*,j))} (1-\bar{\gamma})Y^{(t)}_{\bar{k}^*}\tilde{X}^{(t)}_{\bar{k}^*}(s) + \gamma(({\bar{k}^*,s}))}  \notag \\
        &\leq \sum_{j=1}^{n_{\bar{k}^*}}\frac{X^{(t)}_{\bar{k}^*}(j)}{\frac{\log{n_{\bar{k}^*}}}{\ol\delta^*}\beta} 
        = \frac{\ol\delta^*}{\beta\log{n_{\bar{k}^*}}}. \label{equation: lemma inside block 3}
    \end{align}
	
	Plugging \Cref{equation: lemma inside block 2}, \Cref{equation: lemma inside block 3} into Lemma~\ref{lem:regretX}. Note that for any $\*x\in \Delta_{n_{\bar k}-1}$, $\Phi_{\bar k}(\*x)\leq 0$. Thus, $D_{\Phi_{\bar k^*}}(\Delta_{n_{\bar k^*-1}}) \leq \max_{\*x\in \Delta_{n_{\bar k}-1}} \abs{\Phi_{\bar k}(\*x)} \leq \log n_{\bar k^*}$. Then we obtain
	\begin{equation*}
		\frac{D_{\Phi_{\bar k^*}(\Delta_{n_{\bar k^*-1}})}}{\eta_{\bar k^*}}+\frac{\eta_{\bar k^*}}{2}\cdot\sum_{t=1}^T \E{\sup_{\*x\in [Q^{(t)}_{\bar{k}^*}, [X^{(t)}_{\bar k^*}]} \|\hat\ell^{(t)}_{\bar k^*}\|_{\grad^{-2}\Phi_{\bar k^*}(\*x)}}\le \frac{\log{n_{\bar{k}^*}}}{\eta_{\bar{k}^*}} + \frac{\eta_{\bar{k}^*}\ol\delta^*}{2\beta \log{n_{\bar{k}^*}}} T.
	\end{equation*}
\end{proof}

\section{Proof of \Cref{thm:realized-upper-bound}}\label{sec:proof-main}
\begin{proof}[Proof of \Cref{thm:realized-upper-bound}]
	Assume that the values of $\eta, \eta_S$ and $\eta_{\bar S}$ satisfy $\min_{i\in[m]}(\hat L^{(t)})'(i)\cdot \max\set{\eta, \eta_S, \eta_{\bar S}}\geq -\frac{1}{4}$ for all $t\in[T]$ (it will be verified later that the values we take indeed satisfy this condition for sufficiently large $T$). If $U_1^{\bar S}\neq \emptyset$: choose $\gamma((\bar k,j))=\frac{4\eta_S}{|U_1^S|}$ for $\bar k\in U_1^S$ and $j=1$; choose $\gamma((\bar k,j))=\frac{4\eta_{\bar S}}{|U_1^{\bar S}|-1}$if $\abs{U_1^{\bar S}}>1$ and if $U_1^{\bar S}=1$, let $\gamma((\bar k,j))=0$ for $\bar k\in U_1^{\bar S}$ and $j=1$. If $U_1^{\bar S}= \emptyset$, let $\gamma((\bar k,j))=0$ for $\bar k\in U_1^{S}$ and $j=1$. Here we omit the superscript $(t)$ since these parameters are time-invariant.

	Plugging Lemma~\ref{lem:realized-regretY} and Lemma~\ref{lem:realized-regretX} into \Cref{thm:general regret}, we obtain
	\begin{align*}
		R_{(\bar k^*,j^*)}(T) &\leq \sum_{t=1}^T\E{\inner{\wh L^{(t)}}{Y^{(t)}-\*e^{[m]}_{\bar k^*}}+\sum_{\bar k\in [m]}\sum_{j\in [n_{\bar k}]}\gamma((\bar k,j)) \right. \\ 
		&\left. \quad\quad+\tp{\inner{\hat\ell^{(t)}_{\bar k^*}}{X^{(t)}_{\bar k^*}-\*e^{[n_{\bar k^*}]}_{j^*}}+\sum_{j\in [n_{\bar k^*}]}\gamma_{\bar k^*}(j)}\*1[\bar k^*\in U_2]} \\
		&\leq \tp{\frac{\log{n_{\bar{k}^*}}}{\eta_{\bar{k}^*}} + \frac{\eta_{\bar{k}^*}\ol\delta^*}{2\beta \log{n_{\bar{k}^*}}} T + \alpha T}\*1[\bar k^*\in U_2]+  \frac{\sqrt{|U_2|}}{\eta}\\
		&\quad\quad  + \frac{\sqrt{|U_1^S|}}{\eta_S}+  \eta T \frac{4\sqrt{\sum_{\bar k\in U_2}(\delta^*_{\bar k})^2}}{\alpha} + \frac{2}{1-\bar {\gamma}}\eta_S T \sqrt{|U_1^S|} + \sum_{\bar k \in U_2}\frac{\delta^*_{\bar k}\beta\log n_{\bar k}}{\ol\delta^*} T\\
		&\quad\quad + \tp{{10}\eta_{\ol S}T + \frac{\log(|U_1^{\ol S}|+1)}{\eta_{\ol S}} + 8\eta T \sqrt{|U_2 |} + 8\eta_S T \sqrt{|U_1^S |} + 4\eta_ST \right. \\
		&\left. \quad\quad+ \eta T \frac{12\sqrt{\sum_{\bar k\in U_2}(\delta^*_{\bar k})^2}}{\alpha} + \frac{6}{1-\bar {\gamma}}\eta_S T \sqrt{|U_1^S|}}\*1[U_1^{\ol S}\neq \emptyset].
	\end{align*}
	Choosing $\eta_{\bar k}=\frac{\sqrt{2\beta}\log n_{\bar k}}{\sqrt{\ol\delta^* T}}$ for $\bar k\in U_2$ 
	 and $\eta_{\ol S}=\tp{\frac{\log (|U_1^{\ol S}|+1)}{10T}}^{\frac{1}{2}}$ if $U_1^{\ol S} \neq \emptyset$, we have
	\begin{align}
		R_{(\bar k^*,j^*)}(T)&\leq \tp{\sqrt{\frac{2T\ol\delta^*}{\beta}} +\alpha T}\*1[k^*\in U_2] + \frac{\sqrt{|U_2|}}{\eta} + \eta T \frac{4\sqrt{\sum_{\bar k\in U_2}(\delta^*_{\bar k})^2}}{\alpha}+ \frac{\sqrt{|U_1^S|}}{\eta_S}
		 \notag \\
		&\quad\quad + \sum_{\bar k \in U_2}\frac{\delta^*_{\bar k}\beta\log n_{\bar k}}{\ol\delta^*} T + \frac{2}{1-\bar {\gamma}}\eta_S T \sqrt{|U_1^S|} + \tp{8\eta T \sqrt{|U_2 |} + 8\eta_S T \sqrt{|U_1^S |} + 4\eta_ST \right. \notag \\
		&\quad\quad \left.+\eta T \frac{12\sqrt{\sum_{\bar k\in U_2}(\delta^*_{\bar k})^2}}{\alpha} + \frac{6}{1-\bar {\gamma}}\eta_S T \sqrt{|U_1^S|}+2\sqrt{10\log(|U_1^{\ol S}|+1) T}}\*1[U_1^{\ol S}\neq \emptyset] \notag \\
		&\leq \tp{\sqrt{\frac{2T\ol\delta^*}{\beta}} +\alpha T}\*1[U_2\neq \emptyset]+ \frac{\sqrt{|U_2|}}{\eta} + \eta T \frac{4\sqrt{\sum_{\bar k\in U_2}(\delta^*_{\bar k})^2}}{\alpha}+ \frac{\sqrt{|U_1^S|}}{\eta_S}
		 \notag \\
		&\quad\quad + \sum_{\bar k \in U_2}\frac{\delta^*_{\bar k}\beta\log n_{\bar k}}{\ol\delta^*} T + \frac{2}{1-\bar {\gamma}}\eta_S T \sqrt{|U_1^S|} + \tp{8\eta T \sqrt{|U_2 |} + 8\eta_S T \sqrt{|U_1^S |}  + 4\eta_ST \right. \notag \\
		&\quad\quad \left.+\eta T \frac{12\sqrt{\sum_{\bar k\in U_2}(\delta^*_{\bar k})^2}}{\alpha} + \frac{6}{1-\bar {\gamma}}\eta_S T \sqrt{|U_1^S|}+2\sqrt{10\log(|U_1^{\ol S}|+1) T}}\*1[U_1^{\ol S}\neq \emptyset].
	   \label{eqn: general regret 1}
	\end{align}
	Now we distinguish between the following cases:
	\begin{enumerate}
		\item $U_2=\emptyset$. In this case, the graph is strongly observable and \Cref{eqn: general regret 1} equals to \begin{align*}
			\frac{\sqrt{|U_1^S|}}{\eta_S} + \frac{2}{1-\ol \gamma}\eta_S T \sqrt{|U_1^S|}+ \tp{2\sqrt{10\log(|U_1^{\ol S}|+1) T}+ 20\eta_S T \sqrt{|U_1^S |} + 4\eta_ST}\*1[U_1^{\ol S}\neq \emptyset].
		\end{align*}
		Choosing $\eta_{S}=\sqrt{\frac{1}{\tp{\frac{2}{1-\ol \gamma}+20\cdot\*1[U_1^{\ol S}\neq \emptyset]}T}}$, we have $R_{(\bar k^*,j^*)}(T)\leq 2\sqrt{2|U_1^S|}T^{\frac{1}{2}}$ if $U_1^{\ol S}= \emptyset$ and $R_{(\bar k^*,j^*)}(T)\leq 4\sqrt{6|U_1^S|}T^{\frac{1}{2}} + 2\sqrt{10\log(|U_1^{\ol S}|+1)} T^{\frac{1}{2}} + T^{\frac{1}{2}}$ if $U_1^{\ol S}\neq \emptyset$.
		\item $U_2\neq \emptyset$ and $U_1^{\ol S}=\emptyset$. In this case the graph is weakly observable possibly with strongly observable parts and if so, all strongly observable arms have self-loops. Since $1-\bar{\gamma}\geq \frac{1}{2}$, choose $\eta=\frac{1}{2}\tp{\frac{|U_2|}{\sum_{\bar{k}\in U_2} (\delta^*_{\bar{k}})^2}}^{\frac{1}{4}}\tp{\frac{\alpha}{T}}^{\frac{1}{2}}$, \Cref{eqn: general regret 1} is at most \begin{align*} 
			&\sqrt{\frac{2T\ol\delta^*}{\beta}} +\alpha T + 4\tp{\frac{T}{\alpha}}^{\frac{1}{2}}\tp{\abs{U_2}\sum_{\bar k\in U_2}(\delta^*_{\bar k})^2}^{\frac{1}{4}} + \sum_{\bar k \in U_2}\frac{\delta^*_{\bar k}\beta\log n_{\bar k}}{\ol\delta^*} T + \frac{\sqrt{|U_1^S|}}{\eta_S} + 4\eta_S T \sqrt{|U_1^S|}.
		\end{align*}
		Choosing $\eta_S=\frac{1}{\sqrt{4T}}$, $\alpha=2^{\frac{2}{3}}\frac{\tp{|U_2|\sum_{\bar{k}\in U_2} (\delta^*_{\bar{k}})^2}^{\frac{1}{6}}}{T^{\frac{1}{3}}}$ and $\beta=\frac{\ol\delta^*}{\tp{2T}^{\frac{1}{3}} \tp{\sum_{\bar{k}\in U_2} \delta_{\bar{k}}^* \log{n_{\bar{k}}}}^{\frac{2}{3}}}$, we have \begin{equation*}
			R_{(\bar k^*,j^*)}(T)\leq 3\cdot 2^{\frac{2}{3}}\tp{|U_2|\sum_{\bar{k}\in U_2} (\delta^*_{\bar{k}})^2}^{\frac{1}{6}} T^{\frac{2}{3}} + \frac{3}{2^{\frac{1}{3}}} \cdot \tp{\sum_{\bar{k}\in U_2} \delta_{\bar{k}}^* \log{n_{\bar{k}}}}^{\frac{1}{3}} T^{\frac{2}{3}} + 4\sqrt{|U_1^S|}T^{\frac{1}{2}}.
		\end{equation*}
		\item $U_2\neq \emptyset$ and $U_1^{\ol S}\neq \emptyset$. The graph is a hybrid of weakly and strongly observable parts and some arms in the strongly observable parts have no self-loops. In this case, since $1-\bar{\gamma}\geq \frac{1}{2}$, \Cref{eqn: general regret 1} equals to \begin{align*}
			&\sqrt{\frac{2T\ol\delta^*}{\beta}} +\alpha T+ \frac{\sqrt{|U_2|}}{\eta} + \eta T \frac{16\sqrt{\sum_{\bar k\in U_2}(\delta^*_{\bar k})^2}}{\alpha}+ \frac{\sqrt{|U_1^S|}}{\eta_S}+ \sum_{\bar k \in U_2}\frac{\delta^*_{\bar k}\beta\log n_{\bar k}}{\ol\delta^*} T
		 	\notag \\
			&\quad\quad   +  8\eta T \sqrt{|U_2 |}+ 24\eta_S T \sqrt{|U_1^S |} +2\sqrt{10\log(|U_1^{\ol S}|+1) T} + 4\eta_ST.
		\end{align*}
		Choosing $\eta=\frac{1}{4}\tp{\frac{|U_2|}{\sum_{\bar{k}\in U_2} (\delta^*_{\bar{k}})^2}}^{\frac{1}{4}}\tp{\frac{\alpha}{T}}^{\frac{1}{2}}$, $\alpha=2^{\frac{4}{3}}\frac{\tp{|U_2|\sum_{\bar{k}\in U_2} (\delta^*_{\bar{k}})^2}^{\frac{1}{6}}}{T^{\frac{1}{3}}}$, $\beta=\frac{\ol\delta^*}{\tp{2T}^{\frac{1}{3}} \tp{\sum_{\bar{k}\in U_2} \delta_{\bar{k}}^* \log{n_{\bar{k}}}}^{\frac{2}{3}}}$ and $\eta_S=\frac{1}{2\sqrt{6T}}$, we have \begin{align*}
			R_{(\bar k^*,j^*)}(T)&\leq 6\cdot 2^{\frac{1}{3}}\tp{|U_2|\sum_{\bar{k}\in U_2} (\delta^*_{\bar{k}})^2}^{\frac{1}{6}} T^{\frac{2}{3}} + \frac{3}{2^{\frac{1}{3}}} \cdot \tp{\sum_{\bar{k}\in U_2} \delta_{\bar{k}}^* \log{n_{\bar{k}}}}^{\frac{1}{3}} T^{\frac{2}{3}} 
			 + 4\sqrt{6|U_1^S|}T^{\frac{1}{2}} \\ 
			 &\quad\quad+2\sqrt{10\log(|U_1^{\ol S}|+1)} T^{\frac{1}{2}} + \frac{4T^{\frac{1}{3}} |U_2|^{\frac{5}{6}}}{2^{\frac{1}{3}}\tp{\sum_{\bar{k}\in U_2} (\delta^*_{\bar{k}})^2}^{\frac{1}{6}}} + \frac{\sqrt{6}}{3}T^{\frac{1}{2}}.
		\end{align*}
	\end{enumerate}
	Then we verify $\min_{i\in[m]}(\hat L^{(t)})'(i)\cdot \max\set{\eta, \eta_S, \eta_{\bar S}}\geq -\frac{1}{4}$ for all $t\in[T]$ when $T$ is sufficiently large. Note that when $U_1^{\bar S}=\emptyset$, $(\hat L^{(t)})'(i)\geq 0$ and it is trivial to have that inequality. Then we consider the situation that $U_1^{\bar S}\neq \emptyset$. When  $U_1^{\bar S}\neq \emptyset$, for $i\in[m]$,
	\begin{align*}
		(\hat L^{(t)})'(i)\geq-c^{(t)}=-\sum_{\bar k\in U_1^{\bar S}}\hat L^{(t)}(\bar k)\cdot Y^{(t)}(\bar k)&\geq -\tp{\max_{\ol k\in U_1^{\ol S}} \hat L^{(t)}(\bar k)}\cdot \sum_{\ol k\in U_1^{\ol S}}Y^{(t)}(\bar k) \\
	 	&\geq -\max_{\ol k\in U_1^{\ol S}} \hat \ell^{(t)}_{\bar k}(1)\\
		 &\geq -\frac{1}{\min_{\bar k\in U_1^{\bar S}} \Pr{\mbox{observe }(\bar k, 1) \mbox{\ in round\ }t}}.
	\end{align*}
	Let $\ol k^{(t)}_{\ol S}\triangleq \argmin_{\bar k\in U_1^{\bar S}} \Pr{\mbox{observe }(\bar k, 1) \mbox{\ in round\ }t}$. Note that 
	\begin{align*}
		\min_{\bar k\in U_1^{\bar S}} \Pr{\mbox{observe }(\bar k, 1) \mbox{\ in round\ }t}&\geq \sum_{\ol k\in [m]}\sum_{j\in[n_{\ol k}]}\gamma^{(t)}((\ol k, j))-\gamma^{(t)}((\ol k^{(t)}_{\ol S}, 1))
		=\bar \gamma^{(t)}-\frac{4\eta_{\bar S}}{{|U_1^{\bar S}|-1}}\*1[|U_1^{\ol S}|>1].
	\end{align*}
	Then by direct calculation, when $T$ is sufficiently large, $\abs{\eta\cdot (\hat L^{(t)})'(i)}=O\tp{{\frac{1}{T^{\frac{1}{6}}}}}$, $\abs{\eta_S\cdot (\hat L^{(t)})'(i)}\leq \frac{1}{4}$ and $\abs{\eta_{\bar S}\cdot (\hat L^{(t)})'(i)}\leq \frac{1}{4}$. Thus, $\min_{i\in[m]}(\hat L^{(t)})'(i)\cdot \max\set{\eta, \eta_S, \eta_{\bar S}}\geq -\frac{1}{4}$.
\end{proof}

\section{Proof of \Cref{thm:adapt-realized-upper-bound}}\label{sec:proof-adapt}
\begin{proof}[Proof of \Cref{thm:adapt-realized-upper-bound}]
	Without loss of generality, we assume each node in the weakly observable part of $G$ has in-degree $1$. If not, for each $\ol k\in U_2$, we cut the edges in $G_{\ol k}$ until the in-degree of every node in $G_{\ol k}$ is $1$. We claim that this operation is applicable since it will only increase the mini-max regret. Thus, the upper bound of this spanning subgraph is always larger that the regret of the original graph.

	The remaining proof is similar with the proof of \Cref{thm:realized-upper-bound}. We choose the same $\eta, \eta_S$ and $\eta_{\bar S}$ as we do in \Cref{sec:proof-main}. For $\bar k\in U_1$, we choose the same global exploration factor in \Cref{thm:realized-upper-bound}.
	
	Then plugging Lemma~\ref{lem:realized-regretY} and Lemma~\ref{lem:adapt-realized-regretX} into \Cref{thm:realized-upper-bound}, we obtain
	\begin{align*}
		R_{(\bar k^*,j^*)}(T) &\leq \tp{\frac{\sqrt{n_{\bar{k}^*}}}{\eta_{\bar{k}^*}} + 2\eta_{\bar{k}^*}T\frac{\ol\delta^*}{\beta } + \alpha T}\*1[\bar k^*\in U_2]+  \frac{\sqrt{|U_2|}}{\eta}\\
		&\quad\quad  + \frac{\sqrt{|U_1^S|}}{\eta_S}+  \eta T \frac{4\sqrt{\sum_{\bar k\in U_2}(\delta^*_{\bar k})^2}}{\alpha} + \frac{2}{1-\bar {\gamma}}\eta_S T \sqrt{|U_1^S|} + \sum_{t=1}^T\sum_{\bar k \in U_2}\sum_{j\in[n_{\bar k}]}\gamma^{(t)}(({\bar{k},j}))\\
		&\quad\quad + \tp{{10}\eta_{\ol S}T + \frac{\log(|U_1^{\ol S}|+1)}{\eta_{\ol S}} + 8\eta T \sqrt{|U_2 |} + 8\eta_S T \sqrt{|U_1^S |} + 4\eta_ST \right. \\
		&\left. \quad\quad+ \eta T \frac{12\sqrt{\sum_{\bar k\in U_2}(\delta^*_{\bar k})^2}}{\alpha} + \frac{6}{1-\bar {\gamma}}\eta_S T \sqrt{|U_1^S|}}\*1[U_1^{\ol S}\neq \emptyset].
	\end{align*}
	Note that 
	\begin{align*}
		\sum_{t=1}^T\sum_{\bar k \in U_2}\sum_{j\in[n_{\bar k}]}\gamma^{(t)}(({\bar{k},j}))&=\sum_{t=1}^T\sum_{\bar k \in U_2}\sum_{j\in[n_{\bar k}]}\frac{x_{\bar{k},j}^*}{\ol \delta^*}\cdot \beta\sum_{(\bar k,i)\in N_{\-{out}}((\bar k,j))} \sqrt{X^{(t)}_{\ol k}(i)}\\
		&=\frac{\beta}{\ol\delta^*} \sum_{t=1}^T\sum_{\bar k \in U_2}\sum_{i\in[n_{\bar k}]}\sqrt{X^{(t)}_{\ol k}(i)}\sum_{(\bar k,j)\in N_{\-{in}}((\bar k,i))}x_{\bar{k},j}^*
		\leq \frac{\beta T\tp{\sum_{\ol k\in U_2}\sqrt{n_{\ol k}}}}{\ol \delta ^*}.
	\end{align*}
	Choose $\eta_{\ol k}=\tp{\frac{\beta\sqrt{n_{\ol k}}}{2T\ol \delta^*}}^{\frac{1}{2}}$ for each $\ol k\in U_2$ and $\eta_{\ol S}=\tp{\frac{\log (|U_1^{\ol S}|+1)}{10T}}^{\frac{1}{2}}$ if $U_1^{\ol S} \neq \emptyset$. Then we have
	\begin{align}
		R_{(\bar k^*,j^*)}(T) &\leq \sqrt{\frac{8T\ol\delta^*\sqrt{n_{\ol k^*}}}{\beta}} +\alpha T + \frac{\sqrt{|U_2|}}{\eta} + \eta T \frac{4\sqrt{\sum_{\bar k\in U_2}(\delta^*_{\bar k})^2}}{\alpha}+ \frac{\sqrt{|U_1^S|}}{\eta_S} \notag \\
		&\quad\quad + \frac{\beta T\tp{\sum_{\ol k\in U_2}\sqrt{n_{\ol k}}}}{\ol \delta ^*} + \frac{2}{1-\bar {\gamma}}\eta_S T \sqrt{|U_1^S|} + \tp{8\eta T \sqrt{|U_2 |} + 8\eta_S T \sqrt{|U_1^S |} + 4\eta_ST \right. \notag \\
		&\quad\quad \left.+\eta T \frac{12\sqrt{\sum_{\bar k\in U_2}(\delta^*_{\bar k})^2}}{\alpha} + \frac{6}{1-\bar {\gamma}}\eta_S T \sqrt{|U_1^S|}+2\sqrt{10\log(|U_1^{\ol S}|+1) T}}\*1[U_1^{\ol S}\neq \emptyset] .\label{eqn: adapt general regret 1}
	\end{align}
	Now we distinguish between the following cases:
	\begin{enumerate}
		\item $U_1^{\ol S}=\emptyset$. In this case the graph is weakly observable possibly with strongly observable parts and if so, all strongly observable arms have self-loops. Since $1-\bar{\gamma}\geq \frac{1}{2}$, choose $\eta=\frac{1}{2}\tp{\frac{|U_2|}{\sum_{\bar{k}\in U_2} (\delta^*_{\bar{k}})^2}}^{\frac{1}{4}}\tp{\frac{\alpha}{T}}^{\frac{1}{2}}$, \Cref{eqn: adapt general regret 1} is at most \begin{align*} 
			&\sqrt{\frac{8T\ol\delta^*\sqrt{n_{\ol k^*}}}{\beta}} +\alpha T + 4\tp{\frac{T}{\alpha}}^{\frac{1}{2}}\tp{\abs{U_2}\sum_{\bar k\in U_2}(\delta^*_{\bar k})^2}^{\frac{1}{4}} 
			 + \frac{\beta T\tp{\sum_{\ol k\in U_2}\sqrt{n_{\ol k}}}}{\ol \delta ^*} + \frac{\sqrt{|U_1^S|}}{\eta_S} + 4\eta_S T \sqrt{|U_1^S|}.
		\end{align*}
		Choosing $\eta_S=\frac{1}{\sqrt{4T}}$, $\alpha=2^{\frac{2}{3}}\frac{\tp{|U_2|\sum_{\bar{k}\in U_2} (\delta^*_{\bar{k}})^2}^{\frac{1}{6}}}{T^{\frac{1}{3}}}$ and $\beta=\frac{2^{\frac{1}{3}}\ol\delta^*n_{\ol k^*}^{\frac{1}{6}}}{T^{\frac{1}{3}}\tp{\sum_{\ol k\in U_2}\sqrt{n_{\ol k}}}^{\frac{2}{3}}}$, we have\begin{equation*}
			R_{(\bar k^*,j^*)}(T)\leq 3\cdot \tp{2\sum_{\ol k\in U_2}\sqrt{n_{\ol k}}}^{\frac{1}{3}}n_{\ol k^*}^{\frac{1}{6}}T^{\frac{2}{3}} +3\cdot 2^{\frac{2}{3}}\tp{|U_2|\sum_{\bar{k}\in U_2} (\delta^*_{\bar{k}})^2}^{\frac{1}{6}} T^{\frac{2}{3}} +4\sqrt{|U_1^S|}T^{\frac{1}{2}}.
		\end{equation*}
	\item $U_1^{\ol S}\neq \emptyset$. The graph is a hybrid of weakly and strongly observable parts and some arms in the strongly observable parts have no self-loops. In this case, since $1-\bar{\gamma}\geq \frac{1}{2}$, \Cref{eqn: adapt general regret 1} equals to \begin{align*}
		&\sqrt{\frac{8T\ol\delta^*\sqrt{n_{\ol k^*}}}{\beta}} +\alpha T+ \frac{\sqrt{|U_2|}}{\eta} + \eta T \frac{16\sqrt{\sum_{\bar k\in U_2}(\delta^*_{\bar k})^2}}{\alpha}+ \frac{\sqrt{|U_1^S|}}{\eta_S}+ \frac{\beta T\tp{\sum_{\ol k\in U_2}\sqrt{n_{\ol k}}}}{\ol \delta ^*} \\
		&\quad\quad   +  8\eta T \sqrt{|U_2 |}+ 24\eta_S T \sqrt{|U_1^S |} +2\sqrt{10\log(|U_1^{\ol S}|+1) T} + 4\eta_ST.
	\end{align*}
	Choosing $\eta=\frac{1}{4}\tp{\frac{|U_2|}{\sum_{\bar{k}\in U_2} (\delta^*_{\bar{k}})^2}}^{\frac{1}{4}}\tp{\frac{\alpha}{T}}^{\frac{1}{2}}$, $\alpha=2^{\frac{4}{3}}\frac{\tp{|U_2|\sum_{\bar{k}\in U_2} (\delta^*_{\bar{k}})^2}^{\frac{1}{6}}}{T^{\frac{1}{3}}}$, $\beta=\frac{2^{\frac{1}{3}}\ol\delta^*n_{\ol k^*}^{\frac{1}{6}}}{T^{\frac{1}{3}}\tp{\sum_{\ol k\in U_2}\sqrt{n_{\ol k}}}^{\frac{2}{3}}}$ and $\eta_S=\frac{1}{2\sqrt{6T}}$, we have \begin{align*}
		R_{(\bar k^*,j^*)}(T)&\leq 6\cdot 2^{\frac{1}{3}}\tp{|U_2|\sum_{\bar{k}\in U_2} (\delta^*_{\bar{k}})^2}^{\frac{1}{6}} T^{\frac{2}{3}} + 3\cdot \tp{2\sum_{\ol k\in U_2}\sqrt{n_{\ol k}}}^{\frac{1}{3}}n_{\ol k^*}^{\frac{1}{6}}T^{\frac{2}{3}} + 4\sqrt{6|U_1^S|}T^{\frac{1}{2}} \\ 
		&\quad\quad+2\sqrt{10\log(|U_1^{\ol S}|+1)} T^{\frac{1}{2}} + \frac{4T^{\frac{1}{3}} |U_2|^{\frac{5}{6}}}{2^{\frac{1}{3}}\tp{\sum_{\bar{k}\in U_2} (\delta^*_{\bar{k}})^2}^{\frac{1}{6}}} + \frac{\sqrt{6}}{3}T^{\frac{1}{2}}.
	\end{align*}
	\end{enumerate}
\end{proof}

\end{document}